\def\year{2022}\relax
\documentclass[letterpaper]{article} % DO NOT CHANGE THIS
\usepackage{aaai22}  % DO NOT CHANGE THIS
\usepackage{times}  % DO NOT CHANGE THIS
\usepackage{helvet}  % DO NOT CHANGE THIS
\usepackage{courier}  % DO NOT CHANGE THIS
\usepackage[hyphens]{url}  % DO NOT CHANGE THIS
\usepackage{graphicx} % DO NOT CHANGE THIS
\usepackage{natbib}  % DO NOT CHANGE THIS AND DO NOT ADD ANY OPTIONS TO IT
\usepackage{caption} % DO NOT CHANGE THIS AND DO NOT ADD ANY OPTIONS TO IT
\DeclareCaptionStyle{ruled}{labelfont=normalfont,labelsep=colon,strut=off} % DO NOT CHANGE THIS
\usepackage{algorithm}
\usepackage{newfloat}
\usepackage{listings}
\floatstyle{ruled}
\newfloat{listing}{tb}{lst}{}
\floatname{listing}{Listing}

\usepackage{graphicx}

\usepackage[]{subcaption}
\usepackage{booktabs}
\usepackage{url}
\usepackage{todonotes}
\usepackage{multirow}
\usepackage{tikz}
\usetikzlibrary{calc}
\usetikzlibrary{decorations.pathmorphing}

\newcommand{\xhdr}[1]{{\noindent\bfseries #1}}
\pgfdeclarelayer{background}
\pgfsetlayers{background,main}

\usetikzlibrary{shapes}
\usetikzlibrary{arrows.meta,backgrounds,automata}
\usetikzlibrary{positioning,shapes,matrix,calc}
\tikzstyle{vertex}=[circle, draw, fill=gray!80!white,thick,scale=1.2]
\tikzstyle{edge}=[draw=black, thick,-]

\usepackage[most]{tcolorbox}

\tcbset{mystyle/.style={
    enhanced, 
    coltitle = black, 
    colframe = black!15,
    attach boxed title to top left={yshift=-.25mm-\tcboxedtitleheight/2, xshift=2ex},
    boxed title style={enhanced, 
        frame code={\draw[black!15, line width=.5mm, fill=white, 
            rounded corners=1.75ex] 
            (frame.south west) rectangle (frame.north east);}, 
        interior empty}
    }}

\newtcbox{\mybox}[2][]{mystyle, #1, title=#2}
\newtcbox{\mytikzbox}[2][]{mystyle, tikz upper, #1, title=#2}

\usepackage{amsmath}
\usepackage{amssymb}
\usepackage{amsthm}
\usepackage{amsfonts}
\usepackage{thmtools}		
\usepackage{mleftright}
\usepackage{stmaryrd}
\usepackage{nicefrac}
\makeatletter
\def\@captype{table}
\makeatother

\theoremstyle{definition}
\newtheorem{theorem}{Theorem}
\newtheorem{assumption}{Assumption}

\newtheorem{proposition}[theorem]{Proposition}

\usepackage{thm-restate}
\usepackage[mathic=true]{mathtools}
\usepackage{fixmath}
\usepackage{siunitx}
\usepackage{color}

\usepackage{pifont}

\usepackage{enumitem}
\setlist[enumerate]{noitemsep, topsep=0.5\topsep}
\setlist[description]{noitemsep, topsep=0.5\topsep}
\setlist[itemize]{noitemsep, topsep=0.5\topsep}
 
\usepackage{setspace}
\usepackage[activate={true,nocompatibility},final,kerning=true,stretch=0, shrink=60]{microtype}
\usepackage{ellipsis}
\usepackage{xspace}

\usepackage[utf8]{inputenc} % allow utf-8 input
\usepackage{url}            % simple URL typesetting
\usepackage{booktabs}       % professional-quality tables
\usepackage{amsfonts}       % blackboard math symbols
\usepackage{nicefrac}       % compact symbols for 1/2, etc.
\usepackage{xcolor}         % colors
\usepackage{comment}

\usepackage{url}            % simple URL typesetting
\usepackage{booktabs}       % professional-quality tables
\usepackage{amsfonts}       % blackboard math symbols
\usepackage{nicefrac}       % compact symbols for 1/2, etc.
\usepackage{xcolor}         % colors

\makeatletter
\def\thmt@refnamewithcomma #1#2#3,#4,#5\@nil{%
	\@xa\def\csname\thmt@envname #1utorefname\endcsname{#3}%
	\ifcsname #2refname\endcsname
	\csname #2refname\expandafter\endcsname\expandafter{\thmt@envname}{#3}{#4}%
	\fi
}
\makeatother
\usepackage[capitalise,noabbrev]{cleveref}   
\usepackage{algpseudocode}
\algdef{SE}{ParForAllLoop}{EndParForAllLoop}[1]{\textbf{parallel for}\(\mbox{#1}\) \textbf{do}}{\textbf{end}}

\newcommand{\new}[1]{\emph{#1}}

\DeclarePairedDelimiter{\norm}{\lVert}{\rVert}

\newcommand{\cC}{\ensuremath{{\mathcal C}}\xspace}
\newcommand{\cD}{\ensuremath{{\mathcal D}}\xspace}

\newcommand{\cO}{\ensuremath{{\mathcal O}}\xspace}

\newcommand{\cV}{\ensuremath{{\mathcal V}}\xspace}

\newcommand{\bbR}{\ensuremath{\mathbb{R}}}
\newcommand{\bbQ}{\ensuremath{\mathbb{Q}}}

\newcommand{\bbZ}{\ensuremath{\mathbb{Z}}}

\newcommand{\RR}{\mathbb{R}}

\newcommand{\NN}{\mathbb{N}}

\newcommand{\oms}{\{\!\!\{}
\newcommand{\cms}{\}\!\!\}}

\newcommand{\trans}{^\mathsf{T}}

\newcommand{\round}[1]{\ensuremath{\lfloor#1\rceil}}

\title{MIP-GNN: A Data-Driven Framework for Guiding Combinatorial Solvers}
\author {
    Elias B.\ Khalil,\equalcontrib\textsuperscript{\rm 1, 2}
    Christopher Morris,\equalcontrib\textsuperscript{\rm 3}
    Andrea Lodi\textsuperscript{\rm 4}
}
\affiliations {
    \textsuperscript{\rm 1}Department of Mechanical \& Industrial Engineering, University of Toronto\\
    \textsuperscript{\rm 2}Scale AI Research Chair in Data-Driven Algorithms for Modern Supply Chains\\
    \textsuperscript{\rm 3}Mila -- Quebec AI Institute, McGill University and RWTH Aachen University\\
    \textsuperscript{\rm 4}CERC, Polytechnique Montr\'eal and Jacobs Technion-Cornell Institute, Cornell Tech and Technion -- IIT\\
    khalil@mie.utoronto.ca, chris@christophermorris.info, andrea.lodi@cornell.edu
}

\begin{document}
\maketitle

\begin{abstract}
Mixed-integer programming (MIP) technology offers a generic way of formulating and solving combinatorial optimization problems. While generally reliable, state-of-the-art MIP solvers base many crucial decisions on hand-crafted heuristics, largely ignoring common patterns within a given instance distribution of the problem of interest. Here, we propose MIP-GNN, a general framework for enhancing such solvers with data-driven insights. By encoding the variable-constraint interactions of a given mixed-integer linear program (MILP) as a bipartite graph, we leverage state-of-the-art graph neural network architectures to predict variable biases, i.e., component-wise averages of (near) optimal solutions, indicating how likely a variable will be set to 0 or 1 in (near) optimal solutions of binary MILPs. In turn, the predicted biases stemming from a single, once-trained model are used to guide the solver, replacing heuristic components. We integrate MIP-GNN into a state-of-the-art MIP solver, applying it to tasks such as node selection and warm-starting, showing significant improvements compared to the default setting of the solver on two classes of challenging binary MILPs. Our code and appendix are publicly available at \url{https://github.com/lyeskhalil/mipGNN}. 
\end{abstract}

\section{Introduction}
\noindent Nowadays, combinatorial optimization (CO) is an interdisciplinary field spanning optimization, operations research, discrete mathematics, and computer science, with many critical real-world applications such as vehicle routing or scheduling; see, e.g.,~\cite{Kor+2012} for a general overview. Mixed-integer programming technology offers a generic way of formulating and solving CO problems by relying on combinatorial solvers based on tree search algorithms, such as branch and cut, see, e.g.,~\cite{Nem+1988,Schr+1999,Ber+2005}. Given enough time, these algorithms find certifiably optimal solutions to \textsf{NP}-hard problems. However, many essential decisions in the search process, e.g., node and variable selection, are based on heuristics~\cite{lodi2013heuristic}. The design of these heuristics relies on intuition and empirical evidence, largely ignoring that, in practice, one often repeatedly solves problem instances that share patterns and characteristics. Machine learning approaches have emerged to address this shortcoming, enhancing state-of-the-art solvers with data-driven insights~\cite{Bengio2018,Cap+2021,kotary2021end}. 

Many CO problems can be naturally described using graphs, either as direct input (e.g., routing on road networks) or by encoding variable-constraint interactions (e.g., of a MILP model) as a bipartite graph. As such, machine learning approaches such as \new{graph neural networks}  (GNNs)~\cite{Gil+2017,Sca+2009} have recently helped bridge the gap between machine learning, relational inputs, and combinatorial optimization~\cite{Cap+2021}. GNNs compute vectorial representations of each node in the input graph in a permutation-equivariant fashion by iteratively aggregating features of neighboring nodes. By parameterizing this aggregation step, a GNN is trained end-to-end against a loss function to adapt to the given data distribution. Hence, GNNs can be viewed as a relational inductive bias~\cite{Bat+2018}, encoding crucial graph structures underlying the CO/MIP instance distribution of interest.

\subsection{Present Work} 
We introduce \new{MIP-GNN}, a general GNN-based framework for guiding state-of-the-art branch-and-cut solvers on (binary) mixed-integer linear programs. Specifically, we encode the variable-constraint interactions of a MILP as a bipartite graph where a pair of variable/constraint nodes share an edge iff the variable has a non-zero coefficient in the constraint; see \cref{overview}. To guide a solver in finding a solution or certifying optimality faster  for a given instance, we perform supervised GNN training to predict \emph{variable biases}~\cite{hsu2008probabilistically}, which are computed by component-wise averaging over a set of near-optimal solutions of a given MILP. Intuitively, these biases encode how likely it is for a variable to take a value of 1 in near-optimal solutions. To tailor the GNN more closely to the task of variable bias prediction, we propagate a ``residual error'', indicating how much the current assignment violates the constraints. Further, the theory, outlined in the appendix, gives some initial insights into the theoretical capabilities of such GNNs architecture in the context of MILPs.

We integrate such trained GNNs into a state-of-the-art MIP solver, namely CPLEX~\cite{CPLEX1210}, by using the GNN's variable bias prediction in crucial tasks within the branch-and-cut algorithm, such as \emph{node selection} and \emph{warm-starting}. On a large set of diverse, real-world binary MILPs, modeling the \emph{generalized independent set problem}~\cite{Col+2017} and a \emph{fixed-charge multi-commodity network flow problem}~\cite{hewitt2010combining}, we show  significant improvements over default CPLEX for the task of node selection, while also reporting promising performance for warm-starting and branching variable selection. This is achieved without any feature engineering, i.e., by  relying purely on the graph information induced by the given MILP.\footnote{We focus on binary problems, but the extension to general MILPs is discussed in the appendix.} 

Crucially, for the first time in this line of research, we use a single, once-trained model for bias prediction to speed up \emph{multiple} components of the MIP branch and cut simultaneously. In other words, we show that learning the bias associated with sets of near-optimal solutions is empirically beneficial to multiple crucial MIP ingredients. 

\begin{figure}[t]
    \centering
     \includegraphics[width=\columnwidth]{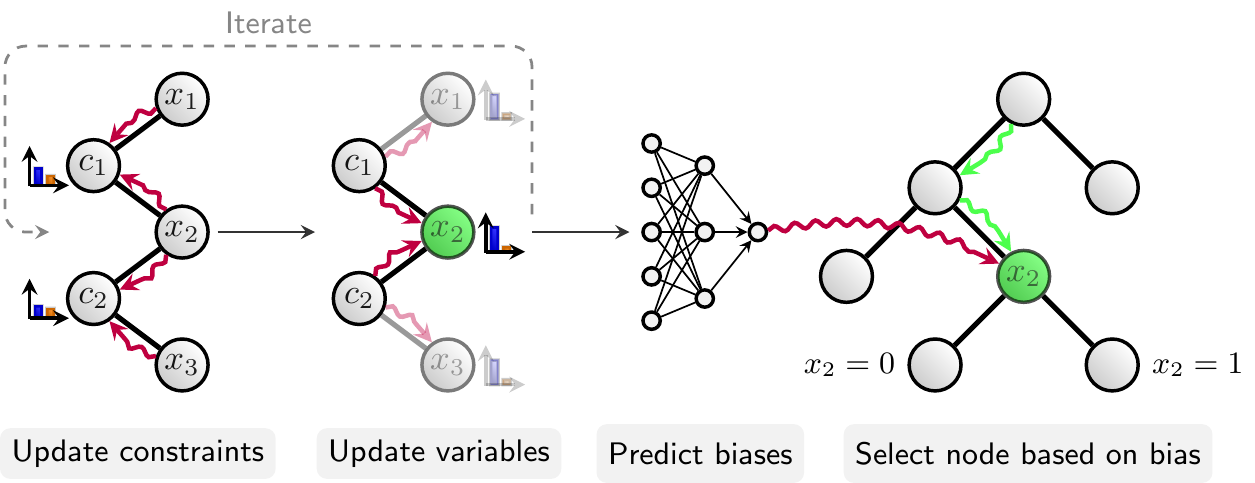}
    \caption{MIP-GNN predicts variables biases for node selection.}    \label{overview}
\end{figure}

\subsection {Related Work} 

\xhdr{GNNs} Graph neural networks (GNNs)~\cite{Gil+2017,Sca+2009} have emerged as a flexible framework for machine learning on graphs and relational data. Notable instances of this architecture include, e.g.,~\cite{Duv+2015,Ham+2017,Vel+2018}, and the spectral approaches proposed in, e.g.,~\cite{Bru+2014,Defferrard2016,Kip+2017}---all of which descend from early work in~\cite{bas+1997,Kir+1995,Mer+2005,mic+2009,mic+2005,Sca+2009,Spe+1997}. Surveys of recent advancements in GNN techniques can be found in~\cite{Cha+2020,Wu+2019}.

\xhdr{Machine learning for CO} \citet{Bengio2018} discuss and review machine learning approaches to enhance combinatorial optimization. Concrete examples include the imitation of computationally demanding variable selection rules within the branch-and-cut framework~\cite{khalil2016,zarpellon2020}, learning to run (primal) heuristics~\cite{khalil2017a,chmiela2021learning}, learning decompositions of large MILPs for scalable heuristic solving~\cite{song2020}, or leveraging machine learning to find (primal) solutions to stochastic integer problems quickly~\cite{bengio2020}. See~\cite{kotary2021end} for a high-level overview of recent advances. 

\xhdr{GNNs for CO} Many prominent CO problems involve graph or relational structures, either directly given as input or induced by the variable-constraint interactions. Recent progress in using GNNs to bridge the gap between machine learning and combinatorial optimization are surveyed in~\cite{Cap+2021}. Works in the field can be categorized between approaches directly using GNNs to output solutions without relying on solvers and approaches replacing the solver's heuristic components with data-driven ones. 

Representative works of the first kind include~\cite{dai2017learning}, where GNNs served as the function approximator for the value function in a Deep $Q$-learning (DQN) formulation of CO on graphs. The authors used a GNN to embed nodes of the input graph. Through the combination of GNN and DQN, a greedy node selection policy is learned on a set of problem instances drawn from the same distribution. \citet{kool2018attention} addressed routing-type problems by training an encoder-decoder architecture, based on Graph Attention Networks~\cite{Vel+2018}, by using an Actor-Critic reinforcement approach. \citet{joshi2019efficient} proposed the use of residual gated graph convolutional networks~\cite{bresson2017residual} in a supervised manner to predict solutions to the traveling salesperson problem. \citet{Fey+2020} and~\citet{Yuj+2019} use GNNs for supervised learning of matching or alignment problems, whereas \citet{Kur+2020,Sel+2019a,Sel+2019} used them to find assignments satisfying logical formulas. Moreover, \citet{Toe+2019} and \citet{Kar+2020} explored unsupervised approaches, encoding constraints into the loss function. 

Works that integrate GNNs in a combinatorial solver conserve its  optimality guarantees. \citet{Gas+2019} proposed to encode the variable constraint interaction of a MILP as a bipartite graph and trained a GNN in a supervised fashion to imitate costly variable selection rules within the branch-and-cut framework of the SCIP solver~\cite{GamrathEtal2020ZR}. Building on that,~\citet{Gupta2020} proposed a hybrid branching model using a GNN at the initial decision point and a light multilayer perceptron for subsequent steps, showing improvements on pure CPU machines. Finally, \citet{nair2020solving} expanded the GNN approach to branching by implementing a GPU-friendly parallel linear programming solver using the alternating direction method of multipliers that allows scaling the strong branching expert to substantially larger instances, also combining this innovation with a novel GNN approach to diving.

Closest to the present work, \citet{Ding2019} used GNNs on a tripartite graph consisting of variables, constraints and a single objective node, enriched with hand-crafted node features. The target is to predict the 0-1 values of the so-called \emph{stable variables}, i.e., variables whose assignment does not change over a set of pre-computed feasible solutions. The predictions then define a ``local branching" constraint~\cite{Fis+2003} which can be used in an exact or heuristic way; the latter restricts the search to solutions that differ in at most a handful of the variables predicted to be ``stable". MIP-GNN differs in the prediction target (biases vs. stable variables), the labeling strategy (leveraging the ``solution pool" of modern MIP solvers vs. primal heuristics), not relying on any feature engineering, flexible GNN architecture choices and evaluation, more robust downstream uses of the predictions via node selection and warm-starting, and tackling much more challenging problem classes.

\section{Preliminaries}

\subsection{Notation}
Let $[n] = \{ 1, \dotsc, n \} \subset \NN$ for $n \geq 1$, and let $\{\!\!\{ \dots\}\!\!\}$ denote a multiset. A \new{graph} $G$ is a pair $(V,E)$ with a \emph{finite} set of \new{nodes} $V$ and a set of \new{edges} $E \subseteq V \times V$. In most cases, we interpret $G$ as an undirected graph. We denote the set of nodes and the set of edges of $G$ by $V(G)$ and $E(G)$, respectively. We enrich the nodes and the edges of a graph with features, i.e., a mapping $l \colon V(G) \cup E(G) \to \mathbb{R}^{d}$. Moreover, $l(x)$ is a \new{feature} of $x$, for $x$ in $V(G) \cup E(G)$. The \new{neighborhood} of $v$ in $V(G)$ is denoted by $N(v) = \{ u \in V(G) \mid (v, u) \in E(G) \}$. A \new{bipartite graph} is a tuple $(A,B,E)$, where $(A \sqcup B,E)$ is a graph, and every edge connects a node in $A$ with a node in $B$. Let $f \colon S \to \mathbb{R}^d$ for some arbitrary domain $S$. Then $\mathbf{f}(s) \in \mathbb{R}^d$,  $s$ in $S$, denotes the real-valued vector resulting from applying $f$ entry-wise to $s$. Last, $[\cdot]$ denotes column-wise (vector) concatenation. 

\subsection{Linear and Mixed-Integer Programs}\label{bip}
A \new{linear program} (LP) aims at optimizing a linear function over a feasible set described as the intersection of finitely many halfspaces, i.e., a polyhedron. We restrict our attention to feasible and bounded LPs. Formally, an instance $I$ of an LP is a tuple $(\mathbf{A},\mathbf{b},\mathbf{c})$, where $\mathbf{A}$ is a matrix in $\bbQ^{m \times n}$, and $\mathbf{b}$ and ${c}$ are vectors in $\bbQ^{m}$ and $\bbQ^{n}$, respectively.
We aim at finding a vector $\mathbf{x}^*$ in $\bbQ^{n}$ that minimizes $\mathbf{c}^T \mathbf{x}^*$ over the \emph{feasible set}
\begin{align*}
	F(I) = \{ \mathbf{x} \in \bbQ^n \mid &\mathbf{A}_j \mathbf{x} \leq b_j \text{ for } j \in [m] \text{ and }  \\
	& x_i \geq 0 \text{ for } i \in [n]  \}.
\end{align*}
In practice, LPs are solved using the Simplex method or (weakly) polynomial-time interior point methods~\cite{Ber+2007}. Due to their continuous nature, LPs are not suitable to encode the feasible set of a CO problem. Hence, we extend LPs by adding \new{integrality constraints}, i.e., requiring that the value assigned to each entry of $\mathbf{x}$ is an integer. Consequently, we aim to find the vector $\mathbf{x}^*$ in $\bbZ^{n}$ that minimizes $\mathbf{c}^T \mathbf{x}^*$ over the feasible set
\begin{align*}
	F_{\textsf{Int}}(I) = \{ \mathbf{x} \in \bbZ^n \mid &\mathbf{A}_j \mathbf{x} \leq b_j  \text{ for } j \in [m] ,  x_i \geq 0, \text{ and }\\ & x_i \in \mathbb{Z} \text{ for } i \in [n]  \},
\end{align*}
and the corresponding problem is denoted as \emph{integer linear program} (ILP).
A component of $\mathbf{x}$ is a \new{variable}, and $\cV(I) = \{x_i \}_{i \in [n]}$ is the set of variables. If we restrict the variables' domains to the set $\{0,1\}$, we get a \emph{binary linear program} (BLP). By dropping the integrality constraints, we again obtain an instance of an LP, denoted $\widehat{I}$, which we call \emph{relaxation}.

\xhdr{Combinatorial solvers} Due to their generality, BLPs and MILPs can encode many well-known CO problems which can then be tackled with branch and cut, a form of tree search which is at the core of all state-of-the-art solving software, e.g.,~\cite{CPLEX1210,GamrathEtal2020ZR,gurobi}. Here, branching attempts to bound the optimality gap and eventually prove optimality by recursively dividing the feasible set and solving LP relaxations, possibly strengthened by cutting planes, to prune away subsets that cannot contain the optimal solution~\cite{Nem+1988}. To speed up convergence, one often runs a number of heuristics or supplies the solver with an initial \new{warm-start solution} if available. Throughout the algorithm's execution, we are left with two main decisions, which node in the tree to consider next (\new{node selection}) and which variable to branch on (\emph{branching variable selection}). 

\subsection{Graph Neural Networks}
Let $G=(V,E)$ be a graph with initial features $f^{(0)} \colon V(G) \to \mathbb{R}^{1\times d}$, e.g., encoding prior or application-specific knowledge. 
A GNN architecture consists of a stack of neural network layers, where each layer aggregates local neighborhood information, i.e., features of neighbors, and then passes this aggregated information on to the next layer.  In each round or layer $t>0$, a new feature $\mathbf{f}^{(t)}(v)$ for a node $v$ in $V(G)$ is computed as
\begin{equation}\label{eq:gnngeneral}
f^{\mathbf{W_2}}_{\text{merge}}\Big(\mathbf{f}^{(t-1)}(v) ,f^{\mathbf{W_1}}_{\text{aggr}}\big(\oms  \mathbf{f}^{(t-1)}(w) \mid  w \in N(v) \cms \big)\Big),
\end{equation}
where $f^{\mathbf{W_1}}_{\text{aggr}}$ aggregates over the set of neighborhood features and $f^{\mathbf{W_2}}_{\text{merge}}$ merges the nodes's representations from step $(t-1)$ with the computed neighborhood features. Both $f^{\mathbf{W_1}}_{\text{aggr}}$ and $f^{\mathbf{W_2}}_{\text{merge}}$ may be arbitrary differentiable functions (e.g., neural networks), while $\mathbf{W_1}$ and $\mathbf{W_2}$ denote sets of parameters.

\section{Proposed Method}\label{mipgnn}

\subsection{Variable Biases and Setup of the Learning Problem}

Let $\cC$ be a set of instances of a CO problem, possibly stemming from a real-world distribution. Further, let $I$  in $\cC$ be an instance with a corresponding BLP formulation $(\mathbf{A},\mathbf{b},\mathbf{c})$. Then, let 
\begin{equation*}
   F^*_{\varepsilon}(I) = \{ \mathbf{x} \in F_{\textsf{Int}}(I) \colon |\mathbf{c}\trans\mathbf{x}^* - \mathbf{c}\trans\mathbf{x}| \leq \varepsilon \} 
\end{equation*}
be a set of feasible solutions that are  close to the  optimal solution $\mathbf{x}^*$ for the instance $I$, i.e., their objective values are within a tolerance $\varepsilon>0$ of the optimal value. The vector of \new{variable biases} $\mathbf{\bar{b}}(I)\in \bbR^n$ of $I$ w.r.t. to $F^*_{\varepsilon}(I)$ is the component-wise average over all elements in $F^*_{\varepsilon}(I)$, namely
\begin{equation}
    \label{eq:bias}
    \mathbf{\bar{b}}(I) = \nicefrac{1}{|F^*_{\varepsilon}|}\sum_{\mathbf{x} \in F^*_{\varepsilon}(I)} \mathbf{x}.
\end{equation}

Computing variable biases is expensive as it involves computing near-optimal solutions. To that end, we aim at devising a neural architecture and training a corresponding model in  a supervised fashion to predict the variable biases $\mathbf{\bar{b}}(I)$ for unseen instances. Letting $\cD$ be a distribution over $\cC$ and $S$ a finite training set sampled uniformly at random from $\cD$, we aim at learning a function $f_{\theta} \colon \cV(I) \to \RR$, where $\theta$ represents a set of parameters from the set $\Theta$, that predicts the variable biases of previously unseen instances. To that end, we minimize the empirical error
\begin{equation*}
\min_{\theta \in \Theta} \nicefrac{1}{|S|}\sum_{I \in S} \ell(\mathbf{f}_{\theta}(\cV(I)),\mathbf{\bar{b}}(I)),
\end{equation*}
with some loss function $\ell \colon \bbR^n \times \bbR^n \to \bbR$ over the set of parameters $\Theta$, where we applied $f_{\theta}$ entry-wise over the set of binary variables $\cV(I)$.

\subsection{The MIP-GNN Architecture}
Next, we introduce the MIP-GNN architecture that represents the function $f_{\theta}$. Intuitively, as mentioned above, we encode a given BLP as a bipartite graph with a node for each variable and each constraint. An edge connects a variable node and constraint node iff the variable has a non-zero coefficient in the constraint. Further, we can encode side information, e.g., the objective's coefficients and problem-specific expert knowledge, as node and edge features. 
Given such an encoded BLP, the MIP-GNN aims to learn an embedding, i.e., a vectorial representation of each variable, which is subsequently fed into a multi-layer perceptron (MLP) for predicting the corresponding bias. To learn meaningful variable embeddings that are relevant to bias prediction, the MIP-GNN consists of two passes, the \new{variable-to-constraint} and the \new{constraint-to-variable} pass; see~\cref{overview}.

In the former, each variable passes its current variable embedding to each adjacent constraint, updating the constraint embeddings. To guide the MIP-GNN in finding meaningful embeddings, we compute an error signal that encodes the degree of violation of a constraint with the current variable embedding. Together with the current constraint embeddings, this error signal is sent back to adjacent variables, effectively propagating information throughout the graph.  

Formally, let $I = (\mathbf{A}, \mathbf{b}, \mathbf{c})$ be an instance of a BLP, which we encode as a bipartite graph $B(I) = (V(I),C(I),E(I))$. Here, the node set $V(I) = \{ v_i \mid x_i \in \cV(I) \}$ represents the variables, the node set $C(I) = \{ c_i \mid i \in [m]  \}$ represents the constraints of $I$, and the edge set $E(I) = \{ \{ v_i,c_j \} \mid A_{ij} \neq 0 \}$ represents their interaction.  Further, we define the (edge) feature function $a \colon E(I) \to \bbR$ as $(v_i,v_j) \mapsto A_{ij}$. Moreover, we may add features encoding additional, problem-specific information, resulting in the feature function $l \colon V(I) \cup C(I)\to \mathbb{R}^d$. In full generality, we implement the variable-to-constraint (v-to-c) and the constraint-to-variable (c-to-v) passes as follows. 

\xhdr{Variable-to-constraint pass} Let $\mathbf{v}^{(t)}_i$ and $\mathbf{c}^{(t)}_j$ in $\mathbb{R}^d$ be the variable embedding of variable $i$ and the constraint embedding of node $j$, respectively, after $t$ c-to-v and v-to-c passes. For $t=0$, we set  $\mathbf{v}^{(t)}_i = l(v_i)$ and $\mathbf{c}^{(t)}_j = l(c_j)$. To update the constraint embedding, we set $\mathbf{c}_j^{(t+1)} =$
\begin{equation*}
 f^{\mathbf{W_{2,C}}}_{\text{merge}}\Big(  \mathbf{c}_j^{(t)}  ,f^{\mathbf{W_{1,C}}}_{\text{aggr}}\big(\oms [\mathbf{v}^{(t)}_i,A_{ji}, b_j]  \mid  v_i \in N(c_j) \cms \big)\Big),
\end{equation*}
where, following~\Cref{eq:gnngeneral},  $f^{\mathbf{W_{1,C}}}_{\text{aggr}}$ aggregates over the multiset of adjacent variable embeddings, and $f^{\mathbf{W_{2,C}}}_{\text{merge}}$ merges the constraint embedding from the $t$-th step with the learned, joint variable embedding representation.

\xhdr{Constraint-to-variable pass} To guide the model to meaningful variable embedding assignments, i.e., those that align with the problem instance's constraints, we propagate \new{error messages}. For each constraint, upon receiving a variable embedding $\mathbf{v}^{(t)}_i$ in $\mathbb{R}^{d}$, we use a neural network $f^{\mathbf{W_a}}_{\text{asg}} \colon \mathbb{R}^{d} \to \mathbb{R}$  to assign a scalar value  $\bar{x}_i = f^{\mathbf{W_a}}_{\text{asg}}(\mathbf{v}^{(t)}_i)$ in $\mathbb{R}$ to each variable, resulting in the vector $\mathbf{\bar{x}}$ in $\mathbb{R}^{|V(I)|}$, and compute the \new{normalized residual}
\begin{equation*}\label{err}
\mathbf{e} = \text{softmax}\big(\mathbf{A} \mathbf{\bar{x}} - \mathbf{b}\big),
\end{equation*}
where we apply a softmax function column-wise, indicating how much the $j$-th constraint, with its current assignment, contributes to the constraints' violation in total. The error signal $\mathbf{e}$ is then propagated back to adjacent variables. That is, to update the variable embedding of node $v_i$, we set $\mathbf{v}_i^{(t+1)} =$
\begin{equation*}
f^{\mathbf{W_{2,V}}}_{\text{merge}}\Big(  \mathbf{v}_i^{(t)}  ,f^{\mathbf{W_{1,V}}}_{\text{aggr}}\big(\oms [\mathbf{c}^{(t)}_j,A_{ji}, b_j, e_j]  \mid  c_j \in N(v_i) \cms \big)\Big).
\end{equation*}
The v-to-c and c-to-v layers are interleaved. 
Moreover, the column-wise concatenation of the variable embeddings over all layers is fed into an MLP, predicting the variable biases. 

MIP-GNN has been described, implemented, and tested on pure binary linear programs, as such problems are already quite widely applicable. However, extensions to general integer variables and mixed continuous/integer problems are possible; see the appendix for a discussion.

\subsection{Simplifying Training}\label{sec:guiding}
Intuitively, predicting whether a variable's bias is closer to $0$ or $1$ is more important than knowing its exact value if one wants to find high-quality solutions quickly. Hence, to simplify training and make predictions more interpretable, we resort to introducing a threshold value $\tau \geq 0$ to transform the original bias prediction---a regression problem---to a classification problem by interpreting the bias $\mathbf{\bar{b}}(I)_i$ of the $i$-th variable of a given BLP $I$ as $0$ if $\mathbf{\bar{b}}(I)_i \leq \tau$ and $1$ otherwise. In the experimental evaluation, we empirically investigate the influence of different choices of $\tau$ on the downstream task. Henceforth, we assume $\mathbf{\bar{b}}(I)_i$ in $\{ 0,1 \}$. Accordingly, the output of the MLP, predicting the variable bias, is a vector $\widehat{\mathbf{p}}$ in $[0,1]^n$.

\subsection{Guiding a Solver with MIP-GNN}

\label{sec:guiding}
Next, we exemplify how MIP-GNN's bias predictions can guide two performance-critical components of combinatorial solvers, node selection and warm-starting; a use case in branching variable selection is discussed in the appendix.

\xhdr{Guided node selection} The primary use case for the bias predictions will be to guide \textit{node selection} in the branch-and-bound algorithm. The node selection strategy will use MIP-GNN predictions to score ``open nodes" of the search tree. If the predictions are good, then selecting nodes that are consistent with them should bring us closer to finding a good feasible solution quickly. To formalize this intuition, we first define a kind of \textit{confidence score} for each prediction as $$\texttt{score}(\widehat{\mathbf{p}}_i)=1-\big|\widehat{\mathbf{p}}_i-\round{\widehat{\mathbf{p}}_i}\big|,$$ where $\round{\cdot}$ rounds to the nearest integer and $\widehat{\mathbf{p}}_i$ is the prediction for the $i$-th binary variable. Predictions that are close to 0 or 1 get a high score (max. 1) and vice versa (min. 0.5). The score of a node is then equal to the sum of the confidence scores (or their complements) for the set of variables that are fixed (via branching) at that node. More specifically, $\texttt{node-score}(N;\widehat{\mathbf{p}})$ is set to
  \begin{equation*}
    \sum_{i\in\text{ fixed-vars}(N)}
    \begin{cases}
      \texttt{score}(\widehat{\mathbf{p}}_i), & \text{if}\ x^N_i=\round{\widehat{\mathbf{p}}_i}, \\
      1-\texttt{score}(\widehat{\mathbf{p}}_i), & \text{otherwise},
    \end{cases}
  \end{equation*}
where $\mathbb{I}\{\cdot\}$ is the indicator function and $x^N_i$ is the fixed value of the $i$-th variable at node $N$. When the fixed value is equal to the rounding of the corresponding prediction, the variable receives $\texttt{score}(\widehat{\mathbf{p}}_i)$; otherwise, it receives the complement, $1-\texttt{score}(\widehat{\mathbf{p}}_i)$. As such, $\texttt{node-score}$ takes into account both how confident the model is about a given variable, and also how much a given node is aligned with the model's bias predictions. Naturally, deeper nodes in the search tree can accumulate larger $\texttt{node-score}$ values; this is consistent with the intuition (folklore) of depth-first search strategies being typically useful for finding feasible solutions.

As an example, consider a node $N_1$ resulting from fixing the subset of variables $x_1=0, x_4=1, x_5=0$; assume the MIP-GNN model predicts biases $0.2, 0.8, 0.9$, respectively. Then, $\texttt{node-score}(N_1; \widehat{\mathbf{p}}) = \texttt{score}(\widehat{\mathbf{p}}_1) + \texttt{score}(\widehat{\mathbf{p}}_4) + (1-\texttt{score}(\widehat{\mathbf{p}}_5))=0.8+0.8+(1-0.9)=1.7$. Another node $N_2$ whose fixing differs only by $x_5=1$ would have a higher score due to better alignment between the value of $x_5$ and the corresponding bias prediction of $0.9$.

While the prediction-guided node selection strategy may lead to good feasible solutions quickly and thus improve the primal bound, it is preferable that the dual bound is also moved. To achieve that, we periodically select the node with the best bound rather than the one suggested by the bias prediction. In the experiments that follow, that is done every 100 nodes.

\xhdr{Warm-starting}
Another use of the bias predictions is to attempt to directly construct a feasible solution via rounding. To do so, the user first defines a \textit{rounding threshold} $p_{\min}$ in $[0.5, 1)$. Then, a variable's bias prediction is rounded to the nearest integer if and only if $\texttt{score}(\widehat{\mathbf{p}}_i) \geq p_{\min}$. With larger $p_{\min}$, fewer variables will be eligible for rounding. Because some variables may have not been rounded, we leverage ``solution repair"\footnote{{\url{https://ibm.com/docs/en/icos/12.10.0?topic=mip-starting-from-solution-starts}}}, a common feature of modern MIP solvers that attempts to complete a partially integer solution for a limited amount of time. Rather than use a single rounding threshold $p_{\min}$, we iterate over a small grid of values $\{ 0.99, 0.98, 0.96, 0.92, 0.84, 0.68\}$ and ask the solver to repair the partial rounding. The resulting integer-feasible solution, if any, can then be returned as is or used to warm-start a branch-and-bound search.

\subsection{Discussion: Limitations and Possible Roadmaps}
In the following, we address limitations and challenges within the MIP-GNN architecture, and discuss possible solutions.

\xhdr{Dataset generation} Making the common assumption that the complexity classes \textsf{NP} and \textsf{co-NP} are not equal,~\citet{Yehuda2020} showed that any polynomial-time sample generator for \textsf{NP}-hard problems samples from an easier sub-problem. However, it remains unclear how these results translate into practice, as real-world instances of CO problems are rarely worst-case ones. Moreover, the bias computation relies on near-optimal solutions, which state-of-the-of-art MIP solver, e.g., CPLEX, can effectively generate but with non-negligible overhead in computing time~\cite{danna2007generating}. In our case, we spend 60 minutes per instance, for example. This makes our approach most suitable for very challenging combinatorial problems with available historical instances that can be used for training.

\xhdr{Limited expressiveness} Recent results, e.g.,~\cite{Maron2019,Morris2019a,Xu2019}, indicate that GNNs only offer limited expressiveness. Moreover, the equivalence between (universal) permutation-equivariant function approximation and the graph isomorphism problem~\cite{Chen2019a}, coupled with the fact that graph isomorphism for bipartite graphs is $\mathsf{GI}$-complete~\cite{Ueh+2005}, i.e., at least as hard as the general graph isomorphism problem, indicate that GNNs will, in the worst-case, fail to distinguish different (non-isomorphic) MILPs or detect discriminatory patterns within the given instances. Contrary to the above negative theoretical results, empirical research, e.g.,~\cite{Gas+2019,nair2020solving,Sel+2019}, as well as the results of our experimental evaluation herein, clearly show the real-world benefits of applying GNNs to bipartite graphs. This indicates a gap between worst-case theoretical analysis and practical performance on real-world distributions.

Nevertheless, in Proposition 1 of the appendix, we leverage a connection to the \emph{multiplicative weights update algorithm}~\cite{Aro+2012} to prove that, under certain assumptions, GNNs are capable of outputting feasible solutions of the underlying BLPs relaxation, minimizing the MAE to (real-valued) biases on a given (finite) training set.

\begin{figure*}[ht!]
\captionsetup[subfigure]{font=large}

   \centering
     \begin{subfigure}[b]{0.47\textwidth}
      \centering
  \includegraphics[scale=0.40]{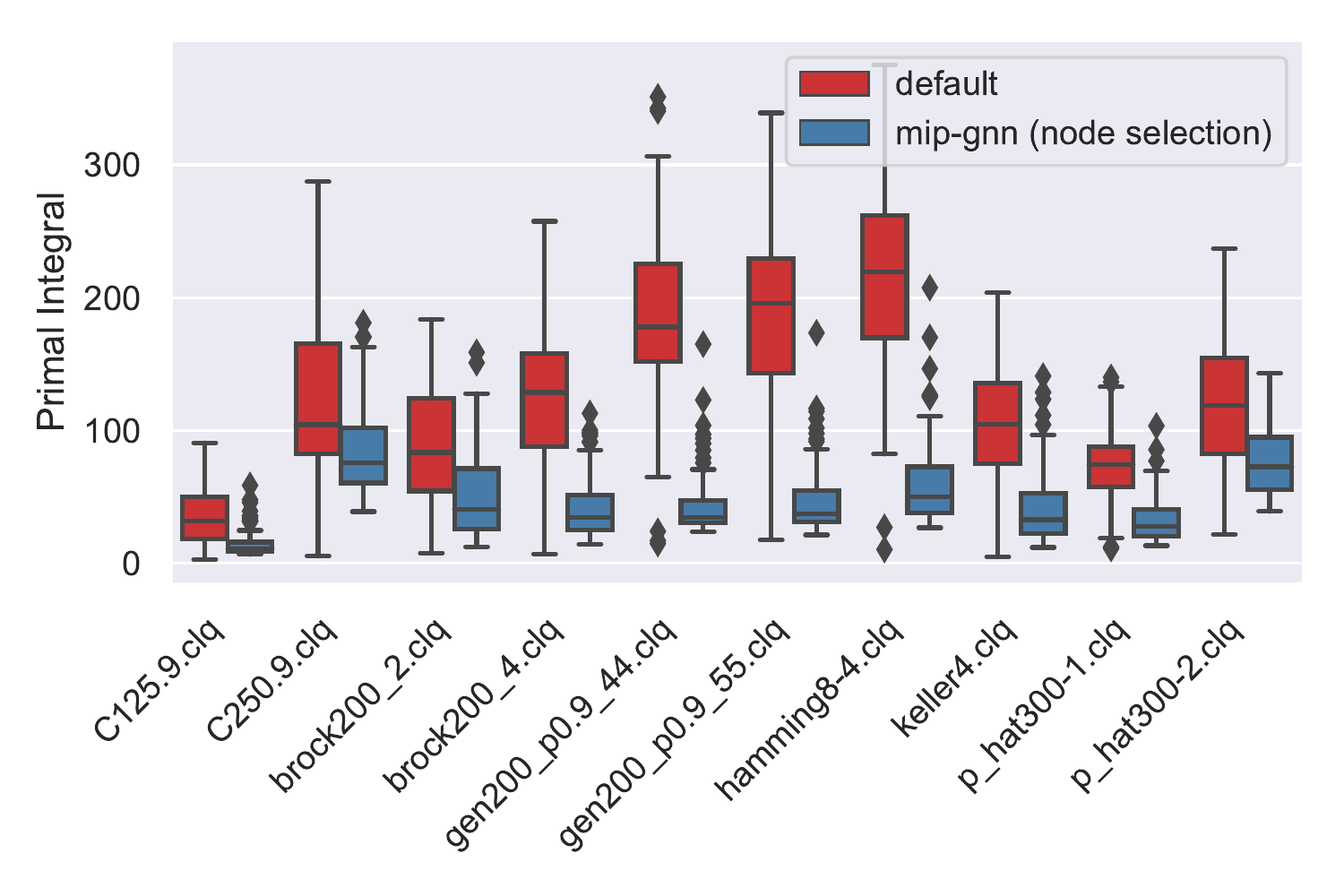}
    \subcaption{Box plots for the distribution of \textbf{Primal Integrals} for the ten problem sets, each with 100 instances; lower is better.}
    \label{fig:gisp_box_primal}
    \end{subfigure}\hspace{10pt}
      \begin{subfigure}[b]{0.47\textwidth}
       \centering
 \includegraphics[scale=0.40]{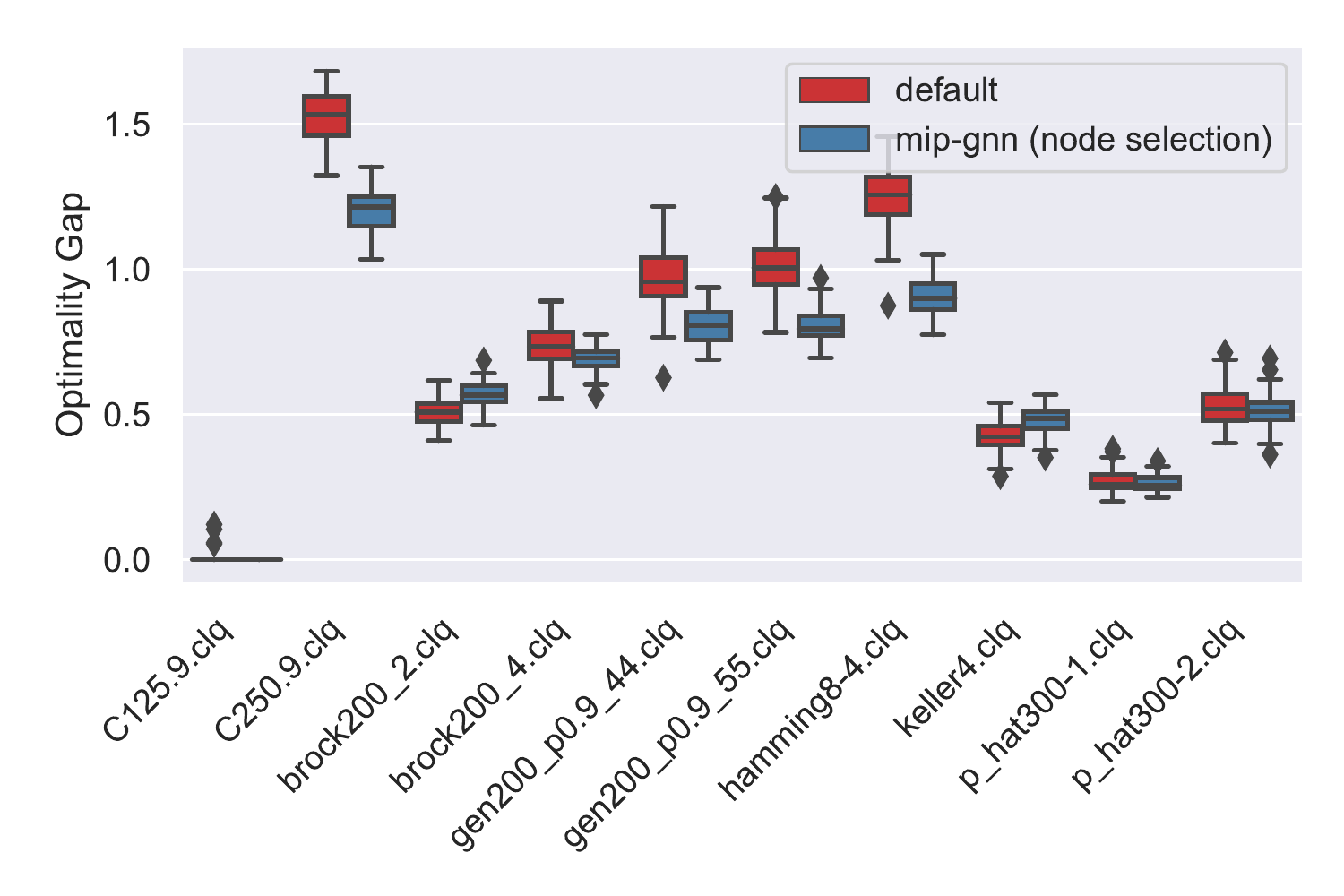}
    \subcaption{Box plots for the distribution of the \textbf{Optimality Gaps} at termination for all problem sets; lower is better.}
    \label{fig:gisp_box_gap}
    \end{subfigure}\vspace{10pt}
 \begin{subfigure}[b]{0.47\textwidth}    
     \centering
    \includegraphics[scale=0.4]{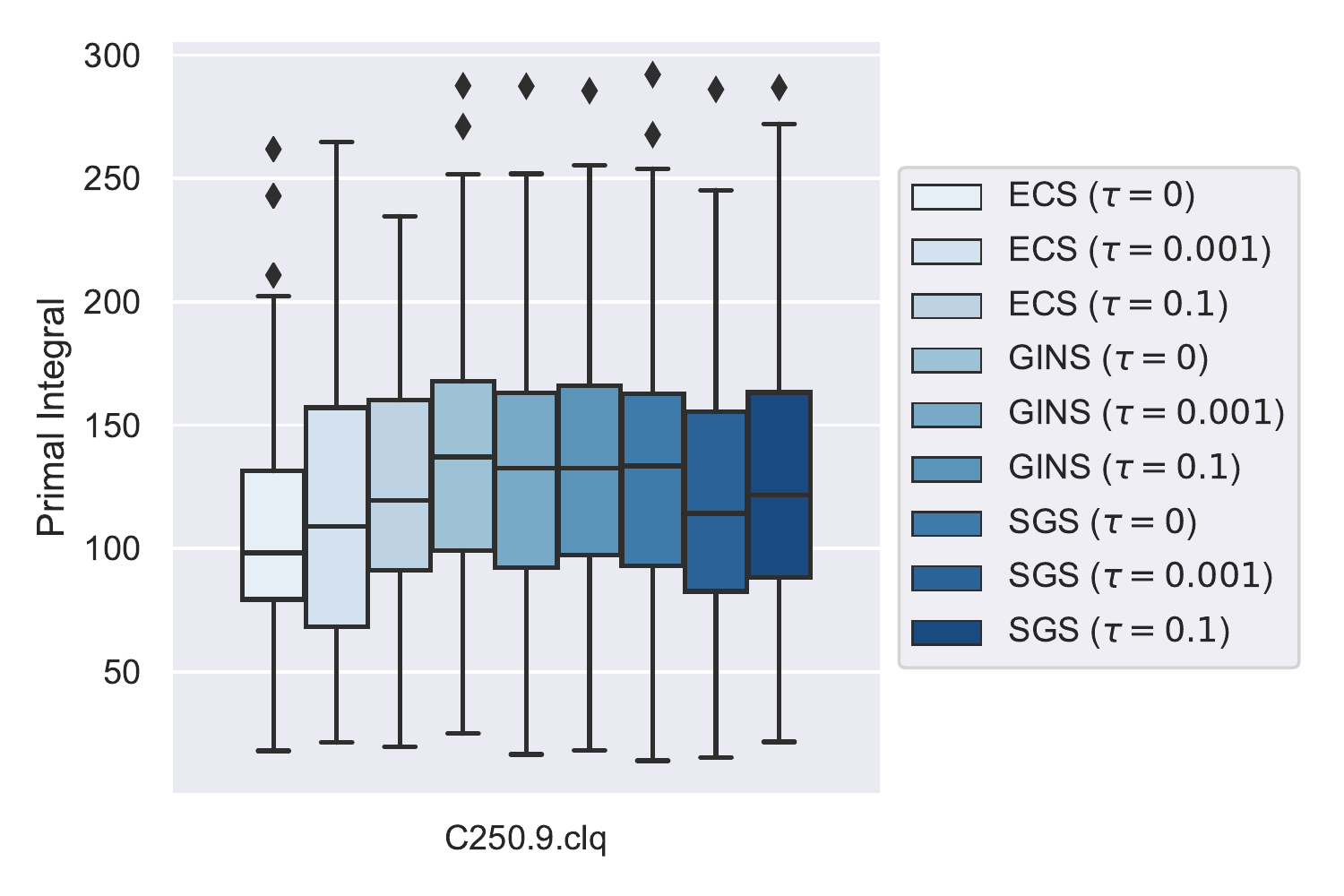}
    \caption{Comparison of three GNN architectures with different $\tau$ values used in training on a single GISP problem set; lower primal integral values are better. The performance impact of the threshold depends on the GNN architecture, with a more pronounced effect for the EdgeConvolution architecture (ECS).}  
    \label{fig:gisp_architectures_box_primal}
     \end{subfigure}\hspace{10pt}
     \begin{subfigure}[b]{0.47\textwidth}
          \centering
    \includegraphics[scale=0.40]{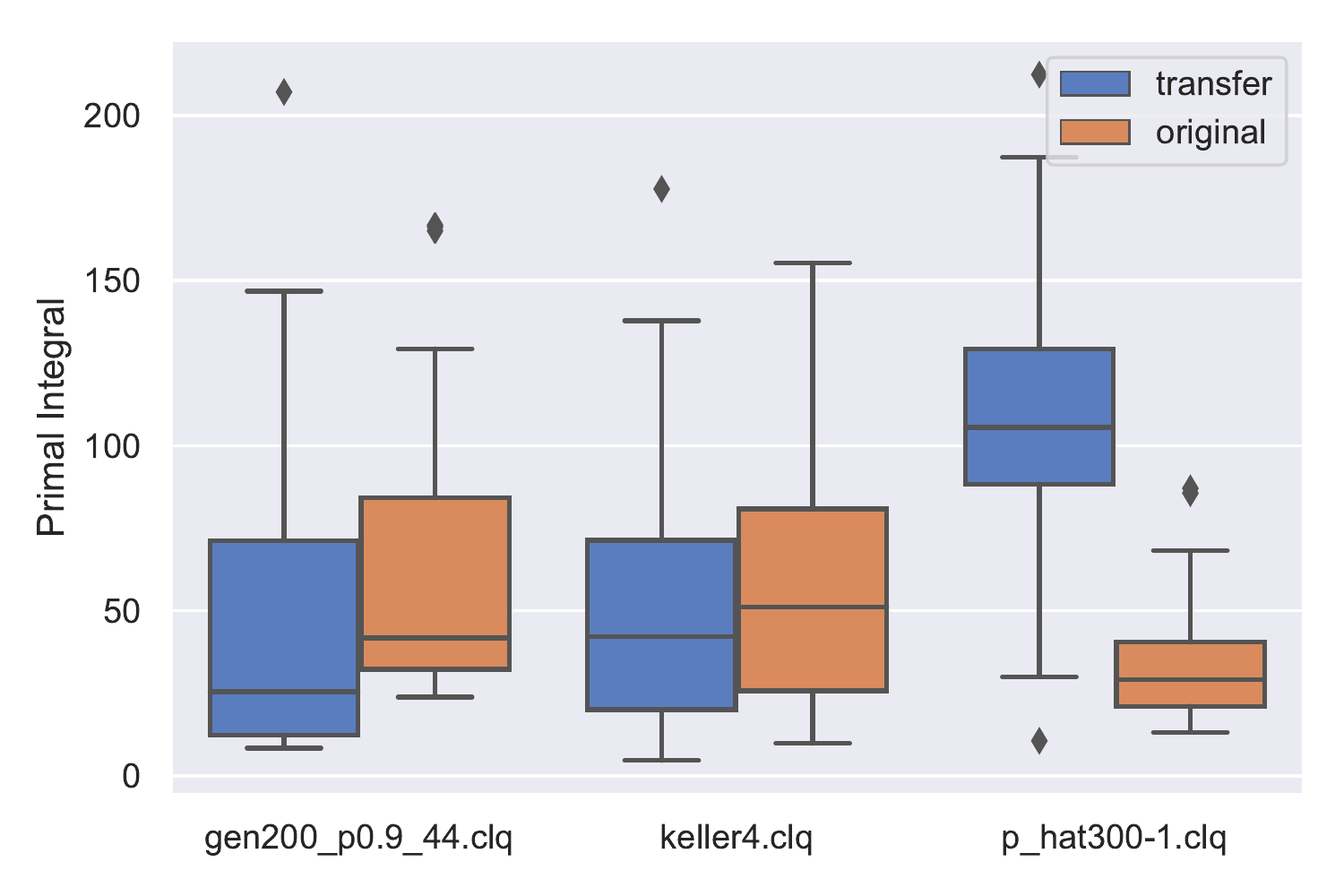}
    \caption{Transfer learning on GISP; Box plots for the distribution of \textbf{Primal Integrals} for three problem sets; lower is better. ``original" refers to the performance of a model trained on instances from the same distribution; ``transfer" refers to that of a model trained on another distribution.}
    \label{fig:gisp_transfer_box_primal}
    \end{subfigure}
    \vspace{5pt}
    \caption{Generalized Independent Set Problem results. }
\end{figure*}

\section{Experimental Evaluation}\label{exp}
Next, we investigate the benefits of using MIP-GNN within a state-of-the-art solver on challenging BLP problems. In this section, we will focus primarily on node selection and to a lesser extent warm-starting. Results on using MIP-GNN to guide variable selection are provided in the appendix. 
We would like to highlight two key features of our experimental design: \textbf{(1)} We use the CPLEX solver, with all of its advanced features (presolve, cuts, heuristics), both as a backend for our method and as a baseline. As such, our implementation and comparisons to CPLEX closely resemble how hard MIPs are solved in real applications, rather than be confined to less advanced academic solvers or artificial case studies; \textbf{(2)} We evaluate our method on two classes of problems that are simultaneously important for operations research applications \textit{and} extremely difficult to find good solutions for or solve to optimality. In contrast, we have attempted to apply MIP-GNN to instances provided by~\citet{Ding2019} and found them to be extremely easy for CPLEX which solves them to global optimality in seconds on average. Because machine learning is unlikely to substantially reduce such already small running times, we believe that tackling truly challenging tasks, such as those we study here, is where learning holds most promise.

\xhdr{Data collection}
Given a BLP training instance, the main data collection step is to estimate the variable biases. To do so, we must collect a set of high-quality feasible solutions. We leverage a very useful feature of modern solvers: the solution pool, particularly CPLEX's implementation.\footnote{{\url{https://ibm.com/docs/en/icos/12.10.0?topic=solutions-what-is-solution-pool}}} This feature repurposes the solver to collect a large number of good integer solutions, rather than focus on proving optimality. For our dataset, we let CPLEX spend 60 minutes in total to construct this solution pool for each instance, terminating whenever it has found $1000$ solutions with objective values within $10\,\%$ of the best solution found, or when the 60-minute limit is reached. The variable biases are calculated according to Eq.~\eqref{eq:bias}.

\xhdr{Neural architecture} To implement the two passes of the MIP-GNN described above, we used the \textsc{GIN-$\varepsilon$}~\cite{Xu2019} (\textsc{GIN}) and GraphSage~\cite{Ham+2017} (\textsc{Sage}) architectures for both passes, with and without error propagation. To deal with continuous edge features, we used a $2$-layer MLP to map them to the same number of components as the node features and combined them using summation. Further, we implemented a variant of EdgeConvolution~\cite{Sim+2017} (EC), again with and without error propagation, handling edge features, in a natural way; see the appendix for details. For node features, we only used the corresponding objective function coefficient and node degree for variable nodes, and the right-hand side coefficient (i.e., the corresponding component of $\mathbf{b}$) and the node degree (i.e., the number of nonzeros in the constraint) for constraint nodes. For edge features, we used the corresponding entry in $\mathbf{A}$.

For all architectures, we used mean aggregation and a feature dimension of $64$. Four GNN layers were used, i.e., four interleaved variable-to-constraint and constraint-to-variable passes, followed by a $4$-layer MLP for the final classification. 

\xhdr{Benchmark problems}
The \emph{generalized independent set problem} (GISP)~\cite{Col+2017} and \emph{fixed-charge multi-commodity network flow problem} (FCMNF)~\cite{hewitt2010combining} have been used to model a variety of applications including forest harvesting~\cite{hochbaum1997forest} and supply chain planning, respectively. Importantly, it is straightforward to generate realistic instances of GISP/FCMNF that are extremely difficult to solve or even find good solutions for, even when using a commercial solver such as CPLEX. 
For each problem set (10 from GISP, 1 from FCMNF), there are 1000 training instances and 100 test instances. These instances have thousands to tens of thousands of variables and constraints, with the GISP problem sets varying in size, as described in Table 2 of~\cite{Col+2017}. Appendix section ``Data generation" includes additional details.

\xhdr{Baseline solver} We use CPLEX 12.10.0 as a backend for data collection and BLP solving. CPLEX ``control callbacks" allowed us to integrate the methods described in Section \ref{sec:guiding} in the solver. We ask CPLEX to ``emphasize feasibility over optimality" by setting its ``emphasis switch" parameter appropriately\footnote{{\url{https://ibm.com/docs/en/icos/12.10.0?topic=parameters-mip-emphasis-switch}}}; this setting makes the CPLEX baseline (referred to as ``default" in the results section) even more competitive w.r.t. evaluation metrics that emphasize finding good feasible solutions quickly. We allow CPLEX to use presolve, cuts, and primal heuristics regardless of whether it is being controlled by our MIP-GNN models or not. As such, all subsequent comparisons to ``default" are to this full-fledged version of CPLEX, rather than to any stripped-down version. We note that this is indeed already a very powerful baseline to compare against, as CPLEX has been developed and tuned over three decades by MIP experts, i.e., it can be considered a very sophisticated \textit{human-learned} solver. The solver time limit is 30 minutes per instance.

\xhdr{Experimental protocol} 
During training, 20\% of the training instances were used as a validation set for early stopping. The training algorithm is \textsc{Adam}~\cite{Kin+2015}, which ran for $30$ epochs with an initial learning rate of $0.001$ and exponential learning rate decay with a patience of $10$. Training is done on GPUs whereas evaluation (including making predictions with trained models and solving MILPs with CPLEX) is done on CPUs with a single thread. Appendix section ``CPU/GPU specifications" provides additional details.

\xhdr{Evaluation metrics} 
All subsequent results will be based on test instances that were not seen during any training run. Because MIP-GNN is designed to guide the solver towards good feasible solutions, the widely used ``Primal Integral" metric~\cite{berthold2013measuring} will be adopted, among others, to assess performance compared to the default solver setting. In short, the primal integral can be interpreted as the average solution quality during a time-limited MIP solve. Smaller values indicate that high-quality solutions were found early in the solving process. Quality is relative to a reference objective value; for GISP/FCMNF, it is typically difficult to find the optimal values of the instances, and so we use as a reference the best solution values found by any of the tested methods. We will also measure the optimality gap. Other relevant metrics will be described in the corresponding figure/table.

\subsection{Results and Discussion}
\xhdr{MIP-GNN (node selection) vs. default CPLEX}
In Figure~\ref{fig:gisp_box_primal}, we use box plots to examine the distribution of primal integral values on the ten test problem sets of GISP. Guiding node selection with MIP-GNN predictions (blue) conclusively outperforms the default setting (red) on all test sets. Equally importantly, appendix Table 5 shows that not only is the primal integral better when using MIP-GNN for node selection, but also that the quality of the best solution found at termination improves almost always compared to default. This improvement on the primal side translates into a reduction of the optimality gap for most datasets, as shown in Figure~\ref{fig:gisp_box_gap}. Additional GISP statistics/metrics are in the appendix. 

As for the FCMNF dataset (detailed results in appendix), MIP-GNN node selection also outperforms CPLEX default, leading to better solutions 81\% of test instances (Table~\ref{tab:fcmnf_wtl_bestval}) and smaller primal integrals on 62\% (Table~\ref{tab:fcmnf_primal}).

Appendix figures~\ref{fig:gisp_allmethods/gisp_box_numsols} and~\ref{fig:gisp_allmethods/gisp_box_numsols_lp} shed more light into how MIP-GNN uses affect the solution finding process in the MIP solver. Strategy ``node selection" finds more incumbent solutions (\cref{fig:gisp_allmethods/gisp_box_numsols}) than ``default", but also many more of those incumbents are integer solutions to node LP relaxations (\cref{fig:gisp_allmethods/gisp_box_numsols_lp}). This indicates that this guided node selection strategy is moving into more promising reasons of the search tree, which makes incumbents (i.e., improved integer-feasible solutions) more likely to be found by simply solving the node LP relaxations. In contrast, ``default" has to rely on solver heuristics to find incumbents, which may incur additional running times.

\xhdr{MIP-GNN (warmstart) vs. default CPLEX}
Appendix Table~\ref{tab:gisp_wtl_bestval_warmstart} shows that warm-starting the solver using MIP-GNN consistently yields better final solutions on 6 out of 9 GISP datasets (with one dataset exhibiting a near-tie); Table~\ref{tab:gisp_wtl_gap_warmstart} shows that the optimality gap is also smaller when warm-starting with our models on 6 out of 9 datasets. This is despite our implementation of warm-starting being quite basic, e.g., the rounding thresholds are examined sequentially rather than in parallel, which means that a non-negligible amount of time is spent during this phase before CPLEX starts branch and bound. We do note, however,  
that guided node selection seems to be the most suitable use of MIP-GNN predictions.

\xhdr{Transfer learning}
Does a MIP-GNN model trained on one set still work well on other, slightly different sets from the same optimization problem? Figure~\ref{fig:gisp_transfer_box_primal} shows the primal integral box plots (similar to those of Figure~\ref{fig:gisp_box_primal}) on three distinct GISP problem sets, using two models: ``original" (orange), trained on instances from the same problem set; ``transfer" (blue), trained on a problem set that is different from all the others. On the first two problem sets, the ``transfer" model performs as well or better than the ``original" model; on the last, ``original" is significantly better. Further analysis will be required to determine the transfer potential.

\xhdr{GNN architectures}
The MIP-GNN results in Figures~\ref{fig:gisp_box_primal} and~\ref{fig:gisp_box_gap} and Table~\ref{tab:gisp_wtl_bestval} used a \textsc{Sage} architecture with error messages and a threshold of $\tau=0$. Figure~\ref{fig:gisp_architectures_box_primal} compares additional GNN architectures with three thresholds on the GISP problem set C250.9.clq (which has the largest number of variables/constraints). The effect of the threshold only affects the EdgeConvolution (ECS) architecture, with zero being the best.

\section{Conclusions}

We introduced MIP-GNN, a generic GNN-based architecture to guide heuristic components within state-of-the-art MIP solvers. By leveraging the structural information within the MILP's constraint-variable interaction, we trained MIP-GNN in a supervised way to predict variable biases, i.e., the likelihood of a variable taking a value of 1 in near-optimal solutions. On a large set of diverse, challenging, real-world BLPs, we showed a consistent improvement over CPLEX's default setting by guiding node selection without additional feature engineering. 
Crucially, for the first time in this line of research, we used a single, once-trained model for bias prediction to speed up \emph{multiple} components of the MIP solver simultaneously. In other words, we showed that learning the bias associated with sets of near-optimal solutions is empirically beneficial to multiple crucial MIP ingredients. We reported in detail the effect on node selection and warm-starting while also showing promising results for variable selection. Further, our framework is extensible to yet other crucial ingredients, e.g., preprocessing, where identifying important variables can be beneficial.

\section*{Acknowledgements}
CM is partially funded by the German Academic Exchange Service (DAAD) through a DAAD IFI postdoctoral scholarship (57515245) and a DFG Emmy Noether grant (468502433). EK is supported by a Scale AI Research Chair.

% Use \bibliography{yourbibfile} instead or the References section will not appear in your paper
% \clearpage
% {
% \fontsize{9.8pt}{10.8pt}
% \selectfont
\bibliography{aaai}
% }

\appendix

\onecolumn

\section{Implementation details}\label{app:code}
Code, data, and pre-trained models can be found at~\url{https://github.com/lyeskhalil/mipGNN}. See the folder \texttt{code/} for the source code of all used GNN architectures and evaluation protocols for the downstream tasks, the folder  \texttt{code/gnn\_models/pretrained\_models/} for pretrained models,
and \texttt{code/README.md} on how to reproduce the reported results.

\section{Data generation}\label{app:datagen}

For GISP, we generate ten problem sets, each corresponding to a different graph that underlies the combinatorial problem; the graphs are from the DIMACS library. The instance generation procedure follows that of~\cite{Col+2017}, particularly the hardest subset from that paper, which is referred to as ``SET2, $\alpha=0.75$". For FCMNF, we generate a single problem set based on the procedure described in~\cite{hewitt2010combining}. 
\cref{tab:ds_stats} includes statistics on the number of variables and constraints in each of the datasets.

\section{CPU/GPU specifications}
All neural architectures were trained on a workstation with two Nvidia Tesla V100 GPU cards with 32GB of GPU memory running Oracle Linux Server 7.7. Training a model took between 15 and 40 minutes depending on the size of the datasets/instances, including inference times on test and validation sets.

When using the trained models within CPLEX, inference is done on CPU and accounted for in all time-based evaluation metrics.
For CPU-based computations (data collection and MILP solving), we used a cluster with 32-core machines (Intel E5-2683 v4) and 125GB memory. Individual data collection or BLP solving tasks were single-threaded, though we parallelized distinct tasks. 

\section{Guiding branching variable selection with MIP-GNN predictions}
As discussed, using MIP-GNN's bias predictions to guide node selection we aim at finding good feasible solutions quickly. Another way of using these predictions is to guide the \textit{variable selection} strategy of the solver. This can be achieved by asking the solver to prioritize branching on variables whose bias prediction is close to 0 or 1. Formally, we provide the solver with priority scores for each binary variable, computed as $\texttt{score}(\widehat{\mathbf{p}}_i)$, the confidence score defined earlier in this section. Then, given a node $N$ with a set $C(N)$ of candidates for branching (namely, the variables with fractional values in the node relaxation solution), the solver branches on $x_i\coloneqq\arg\max_{i\in C(N)}\{\texttt{score}(\widehat{\mathbf{p}}_i)\}$.

Note that the node selection and branching cases use the predicted bias in a complementary way: for node selection, we look at the variables already fixed, while, for variable selection, we look at those yet to be fixed.

\section{Extension to ILP and MILP}
\label{app:extension}
While we described the MIP-GNN as an architecture for BLPs, it applies to general ILPs as well. That is, we can transform a given ILP instance into a BLP through a binary expansion of each integer (non-binary) variable $x_i$ by assuming an upper bound $U_i$ on it is known or it can be computed (as it is generally the case). This clearly requires the addition of 
$\lceil \log_2 U_i \rceil + 1$ binary variables, which can lead to a large BLP if the $U_i$'s are big. 
Alternatively, we can stick to the given ILP formulation and treat the learning problem as a multi-class instead of a binary classification problem. 

Finally, in the presence of continuous variables, i.e., either MILPs or mixed-binary linear programs, we ignore them and compute the bias only for the integer-constrained variables by assuming that the optimal value of the continuous variables can be computed \emph{a posteriori} by simpling solving an LP (i.e., after fixing the integer variables to their predicted values).

\section{Details on GNN architectures}\label{gnn_expanded}

As outlined in the main paper, we used the \textsc{GIN-$\varepsilon$} (\textsc{GIN})~\cite{Xu2019},  GraphSage~\cite{Ham+2017}, with and without error propagation, to express $f^{\vec{W_{2,X}}}_{\text{merge}}$ and $f^{\vec{W_{1,X}}}_{\text{aggr}}$ for $\vec{X}$ in $\{ \vec{C}, \vec{V} \}$. To deal with continuous edge features, we used a $2$-layer MLP to map them to the same number of components as the node features and combined them using summation. Further, we implemented a variant of EdgeConvolution (EC)~\cite{Sim+2017}, that can handle continuous edge feature in a natural way. That is, a feature $\vec{v}_i^{(t+1)}$ (c-to-v pass) for variable node $v_i$ is computed  as 
\begin{align*}
\sum_{j \in \mathcal{N}(i)}
        h([\vec{v}_i^{(t)}, \vec{c}_j^{(t)}, A_{ji}, b_j, e_j]), 
\end{align*}
where $h$ is a $2$-layer MLP. The  v-to-c pass is computed in an analogous fashion.

\section{Theoretical guarantees}\label{prop}

In the following, we show that, under certain assumptions, GNNs are, in principle, capable of outputting feasible solutions of the relaxation of the underlying BLPs, minimizing the MAE to (real-valued) biases on a given (finite) training set. We do so by leveraging a connection to the \emph{multiplicative weights update algorithm} (MWU), e.g., see~\cite{Aro+2012} for a thorough overview. 

Given a BLP $I = (\vec{A},\vec{b},\vec{c})$, its corresponding relaxation $\widehat{I}$, and $\varepsilon > 0$, a vector $\vec{x}_{\varepsilon}$ in $\mathbb{R}^n$ is a \emph{$\varepsilon$-feasible solution} for the relaxation $\widehat{I}$ if
\begin{align*}
    \vec{A} \vec{x}_{\varepsilon} \geq  \vec{b} - \vec{1}\varepsilon
\end{align*}
holds.\footnote{For technical reasons, without loss of generality, we flipped the relation here. That is, a feasible solution of the relaxation $\widehat{I}$ satifies $\vec{A}_j \vec{x} \geq b_j \text{ for } j \in [m]$.} That is, an $\varepsilon$-feasible solution violates each constraint of $\widehat{I}$ by at most $\varepsilon$. 

Let $\vec{a}$ in $\mathbb{R}^n$ and $b$ in $\mathbb{R}$, then we assume, following, e.g.,~\cite{Aro+2012}, that there exists an oracle $\mathcal{O}$ either returning $\vec{x} \geq 0$ satisfying $\vec{a}\trans \vec{x} \geq b$ or determining that the system is not feasible. Put differently, we assume that we have access to an oracle $\mathcal{O}$ that outputs a feasible solution $\vec{x} \geq 0$ to a linear system with one inequality, i.e., $\vec{p}\trans \vec{A}\vec{x} \geq \vec{p}\trans\vec{b}$, where the vector $\vec{p}$ is from the $(n+1)$-dimensional probability simplex. Again following, e.g.,~\cite{Aro+2012},
we assume that for any $\vec{x}$ in $\mathbb{R}^n$ the oracle $\cO$ outputs it holds that
\begin{align*}
    \vec{A}_j \vec{x} - b_j \in [-\rho, \rho]
\end{align*}
for some $\rho > 0$ and $j$ in $[m]$, resulting in the following assumption.
\begin{assumption}\label{assum}
	There exists a neural unit simulating the oracle $\mathcal{O}$, i.e., returning the same results as the oracle $\mathcal{O}$, given the same input.
\end{assumption}
Although this assumption is rather strict, such oracle can be expressed by standard neural architectures, see, e.g.,~\cite{Per+2019}. Further, given a BLP $I$ or a finite set of BLP, we assume an upper-bound of $\Delta \cdot n > 0$ on the minimum $\ell_1$ distance between a feasible point of the relaxation $\widehat{I}$ and the bias vector $\vec{\bar{b}}(I)$ of the BLP $I$. 

\begin{assumption}\label{gap}
Given a BLP $I = (\vec{A},\vec{b},\vec{c})$ with bias $\vec{\bar{b}}(I)$ and  corresponding relaxation $\widehat{I}$, there exists a $\Delta > 0$ such that 
\begin{align*}
	\min_{\vec{x} \in F(\hat{I})} \norm{ \vec{x} - \vec{\bar{b}}(I) }_1 \leq \Delta \cdot n.
\end{align*}
For a set of BLP $S$, we assume 
\begin{align*}
	\sum_{I \in S} \min_{\vec{x} \in F(\hat{I})} \norm{ \vec{x} - \vec{\bar{b}}(I) }_1 \leq \Delta \cdot \sum_{I \in S} n_i, 
\end{align*}
where $n_{i}$ denotes the number of variables of instance $I_i$.
\end{assumption}

Now the following result states that there exists a GNN architecture and corresponding weight assignments that outputs $\varepsilon$-feasible solutions for the relaxations of MILPs from a finite (training) set, minimizing the MAE to (real-valued) bias predictions.

\begin{proposition}\label{bound}
Let $S = \{ (I_i, \vec{\bar{b}}(I_i)) \}_{i \in [|S|]}$ be a finite training set of (feasible) MILPs with corresponding biases and let $\varepsilon > 0$. Then, under Assumption 1 and 2, there exists a GNN architecture and corresponding weights such that it outputs an $\varepsilon$-feasible solution $\vec{\widehat{x}}_{I_i}$ for the relaxation of each MILP  $I_i$ in $S$ and
\begin{equation*}
    \nicefrac{1}{|S|} \sum_{I_i \in S} \nicefrac{1}{n_i} \norm{ \vec{\widehat{x}}_{I_i} - \vec{\bar{b}}(I_i) }_1 \leq \varepsilon + \Delta,
\end{equation*}
upper-bounding the MAE on the training set $S$. Here, $n_i$ denotes the number of variables of instance $I_i$.  The number of layers of the GNN is polynomially-bounded in $\varepsilon$ and logarithmically-bounded in the number of constraints over all instances. 
\end{proposition}

\begin{proof}[Proof sketch]
Without loss of generality, we can encode the $S$ relaxations $\widehat{I}_{i \in [|S|]}$ as one LP with a sparse, block-diagonal constraint matrix. Henceforth, we just focus on a single instance/bias pair $(I,\bar{\vec{b}})$ in $S$ with $I = (\vec{A}, \vec{b}, \vec{c})$. Now consider the following (convex) optimization problem: 
\begin{align}
    \tag{L1}
    \arg\min_{\vec{x}} \norm{\vec{x} - \vec{\bar{b}}}_1\\
    \vec{A}\vec{x} \geq \vec{b}.\nonumber
\end{align}
Since, by assumption, $\norm{\vec{x} - \vec{\bar{b}}}_1 = \sum_{i=1}^n | x_i - \vec{\bar{b}}_i | \leq \Delta \cdot n$, each summand, ``on average'', introduces an error of $\Delta$.
Hence,  we can transform the optimization problem L1, using a standard transformation, into a system of linear inequalities (System L2). Since system L2 is feasible by construction, the MWU algorithm will return an $\varepsilon$-feasible solution  $\vec{\widehat{x}}$ for L2 and the relaxation $\widehat{I}$, see~\cref{arora}. By stitching everything together, we get
\begin{align*}
    \norm{\vec{\widehat{x}} - \vec{\bar{b}}}_1 \leq (\Delta + \varepsilon) \cdot n. 
\end{align*}
Hence, on average over $n$ variables, we introduce an error of $\varepsilon + \Delta$. This finishes the part of the proof upperbounding the MAE. 

We now show how to simulate the MWU via a GNN. First, observe that we can straightforwardly interpret the MWU algorithm as a message-passing algorithm, as outlined in~\cref{alg:as}. Since every step can be expressed as a summation or averaging over neighboring nodes on a suitably defined graph, encoding the system L2 for all $S$ relaxations $\widehat{I}_{i \in [|S|]}$, there exists an implementation of the functions $f^{t,\vec{W_1}}_{\text{aggr}}$ and $f^{t,\vec{W_2}}_{\text{merge}}$ of~\cref{eq:gnngeneral} and corresponding parameter initializations such that we can (exactly) simulate each iteration $t$ of the MWU in $\mathbb{R}^d$ for some suitable chosen $d>0$, depending linearly on $|S|$ and $n$. By the same reasoning, there exists an MLP that maps each such encoded variable assignment for variable $x_i$ onto the solution $\nicefrac{\bar{x}_i}{T}$ returned by~\cref{alg:as}. The bound on the number of required GNN layers is a straightforward application of~\cref{arora}.%\cm{finish this}
\end{proof}

\begin{algorithm}[H]\mbox{\hfill}
	\\\textbf{Input:} Bipartite constraint graph $B(I)$, $\varepsilon > 0$,  stepsize $\eta > 0$, scaling constant $\rho$. \\
	\textbf{Output:} $\varepsilon$-feasible $\vec{\bar{x}}$ or \texttt{non-feasible}.
	\begin{algorithmic}[1]
		\State Initialize weights $w_j \leftarrow 1$ for each constraint node
		\State Initialize probabilities $p_j \leftarrow \nicefrac{1}{m}$ for each constraint node
		\State Set $T$ to according~\cref{mwubound}
		\For{$t \text{ in } [T]$}
		\State Update each variable node $\vec{v}_i$ based on output $\vec{x}_i$ of oracle $\mathcal{O}$ using $\vec{p}$
		\State Compute error signal ${e}_i$ for each constraint node $\vec{c}_j$
		\begin{equation*}
		e_j \leftarrow \nicefrac{1}{\rho}\,\Big(\sum_{i \in N(j)} \label{key} A_{ji}  \vec{x}_i \Big) - \vec{b}_j   
		\end{equation*}
		\State Update $w_i \leftarrow (1 - \eta e_j) w_i$ 
		\State Normalize weights $p_j \leftarrow \nicefrac{w_j}{\boldsymbol{\Gamma}(t)} $  for $i$ in $[m]$, where $\boldsymbol{\Gamma}(t) = \sum_{i \in [m]} \vec{w}_i$
		\State Update solution $\vec{\bar{x}} \leftarrow \vec{\bar{x}} + \vec{x}$
		\EndFor
		\State \Return Average over solutions $ \nicefrac{\vec{\bar{x}}}{T}$
	\end{algorithmic}
	\caption{MWU (Message passing version) for the LP feasibility problem.}
	\label{alg:as}
\end{algorithm}

\begin{theorem}[E.g., \cite{Aro+2012}]\label{arora}
Given a BLP $I = (\vec{A},\vec{b},\vec{c})$, its corresponding relaxation $\widehat{I}$, and $\varepsilon > 0$,~\cref{alg:as} with 
\begin{equation}\label{mwubound}
	T = \left\lceil \frac{4 \rho \ln(m)}{\varepsilon^2} \right\rceil,
\end{equation}
outputs an $\varepsilon$-feasible solution for $\widehat{I}$ or determines that $\widehat{I}$ is not feasible.
\end{theorem}

\section{Additional experimental results -- FCMNF}
\label{app:fcmnf}

Here, we report on additional results for the FCMNF dataset. In summary, while the improvements due to guiding the solver with MIP-GNN predictions are more pronounced for GISP, they are also statistically significant for FCMNF.
\begin{itemize}
    \item[--] Table~\ref{tab:fcmnf_wtl_bestval} shows that MIP-GNN node selection leads to better solutions than the default solver on 81/100 test instances of this problem.
    
    \item[--] Similarly, Table~\ref{tab:fcmnf_primal} shows that the primal integral is smaller (better) most of the time (62/100 wins). Table~\ref{tab:fcmnf_primal} shows that the primal integral is also smaller on average and w.r.t.\@ the median.
    
    \item[--] Figures~\ref{fig:fcmnf_main/fcmnf_box_primal} and~\ref{fig:fcmnf_main/fcmnf_box_gap} are analogous to GISP Figures~\ref{fig:gisp_box_primal} and~\ref{fig:gisp_box_gap} from the main text.

\end{itemize}

\begin{table}[!htbp]
\centering
\caption{FCMNF; number of wins, ties, and losses, for the proposed method (MIP-GNN for node selection), with respect to the \textbf{objective value of the best solution found} compared to the default solver setting.}
\label{tab:fcmnf_wtl_bestval}
\begin{tabular}{lrrrr}
\toprule
\textbf{Problem Set} & \textbf{Wins} & \textbf{Ties} & \textbf{Losses} &     \textbf{p-value} \\
\midrule
L\_n200\_p0.02\_c500 &   81 &    0 &     17 & 5.18129e-14 \\
\bottomrule
\end{tabular}
\end{table}

\begin{table}[!htbp]
\caption{FCMNF; number of wins, ties, and losses, for the proposed method (MIP-GNN for node selection), with respect to the \textbf{Primal Integral} compared to the default solver setting.}
\centering
\label{tab:fcmnf_wtl_primalintegral}
\begin{tabular}{lrrrr}
\toprule
\textbf{Problem Set} & \textbf{Wins} & \textbf{Ties} & \textbf{Losses} &     \textbf{p-value} \\
\midrule
L\_n200\_p0.02\_c500 &   62 &    0 &     36 & 0.015951 \\
\bottomrule
\end{tabular}
\end{table}

\begin{table}[!htbp]
\caption{FCMNF; statistics on the \textbf{Primal Integral} for the one problem sets with 100 instances; lower is better.}
\centering
\label{tab:fcmnf_primal}
\begin{tabular}{llrrr}
\toprule
\multirow{3}{*}{\vspace*{4pt}\textbf{Problem Set}}& \multirow{3}{*}{\vspace*{4pt}\textbf{Method}} &\multicolumn{3}{c}{\textbf{Statistics}}\\
				\cmidrule{3-5}
               &             & \textbf{Avg.}          &  \textbf{Std.} & \textbf{Median}  \\
\midrule
L\_n200\_p0.02\_c500 & default & 137.05 & 49.74 & 133.46 \\
                  & MIP-GNN (node selection) & 125.97 & 34.14 & 122.49 \\
\bottomrule
\end{tabular}
\end{table}

\begin{table}[!htbp]
\caption{FCMNF; statistics on the \textbf{Optimality Gap} for the one problem sets with 100 instances; lower is better.}
\centering
\label{tab:fcmnf_gap}
\begin{tabular}{llrrr}
\toprule
\multirow{3}{*}{\vspace*{4pt}\textbf{Problem Set}}& \multirow{3}{*}{\vspace*{4pt}\textbf{Method}} &\multicolumn{3}{c}{\textbf{Statistics}}\\
				\cmidrule{3-5}
               &             & \textbf{Avg.}          &  \textbf{Std.} & \textbf{Median}  \\
\midrule
L\_n200\_p0.02\_c500 & default & 0.33 & 0.03 &   0.33 \\
                  & MIP-GNN (node selection) & 0.35 & 0.03 &   0.34 \\
\bottomrule
\end{tabular}
\end{table}

\begin{figure}[!htbp]
\centering
\begin{subfigure}[b]{0.45\textwidth}
\includegraphics[scale=0.40]{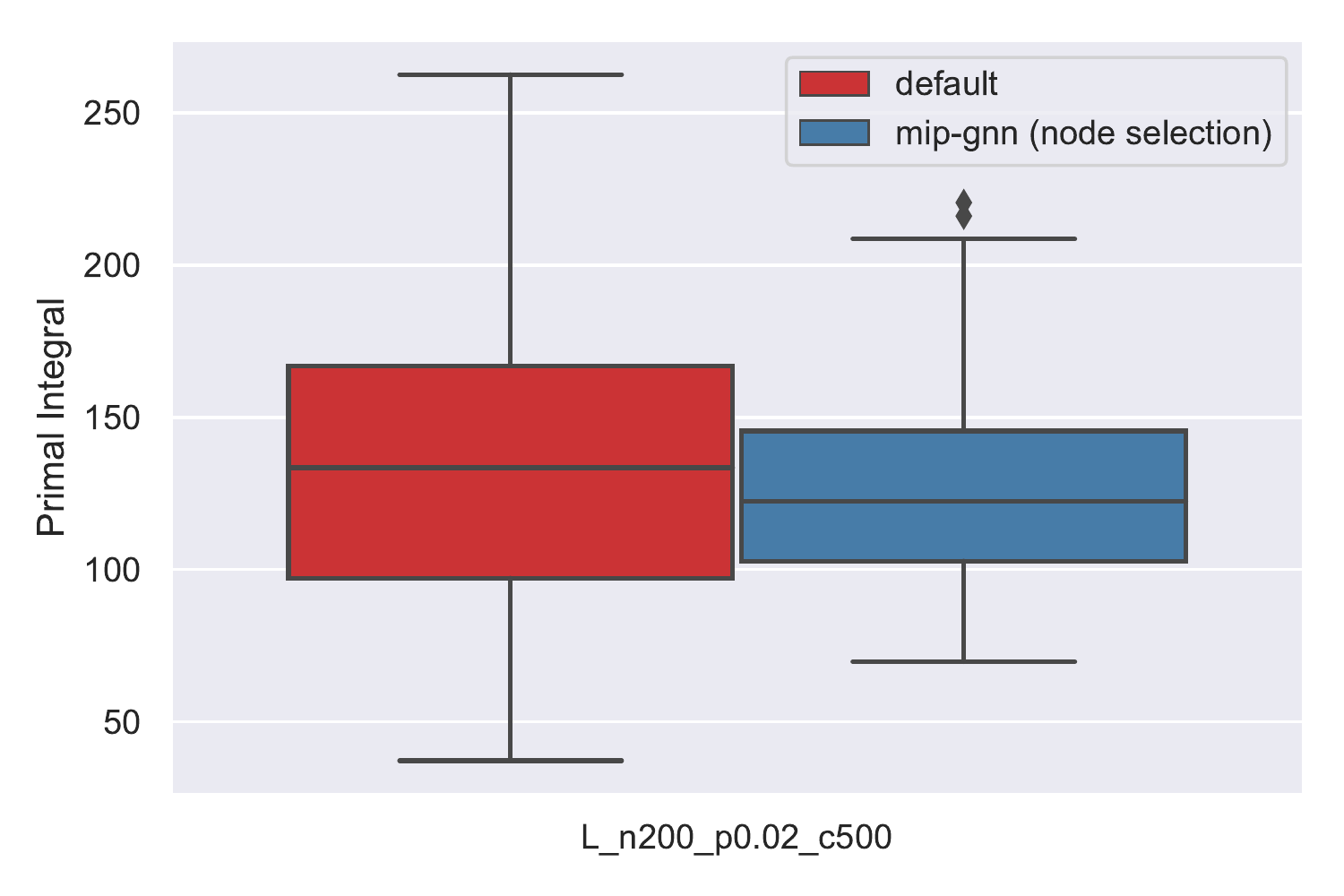}
\caption{Box plots for the distribution of \textbf{Primal Integrals} for the ten problem sets, each with 100 instances; lower is better. \vspace{10pt}}
    \label{fig:fcmnf_main/fcmnf_box_primal}
\end{subfigure}\hspace{10pt}
\begin{subfigure}[b]{0.45\textwidth}
\centering
\includegraphics[scale=0.40]{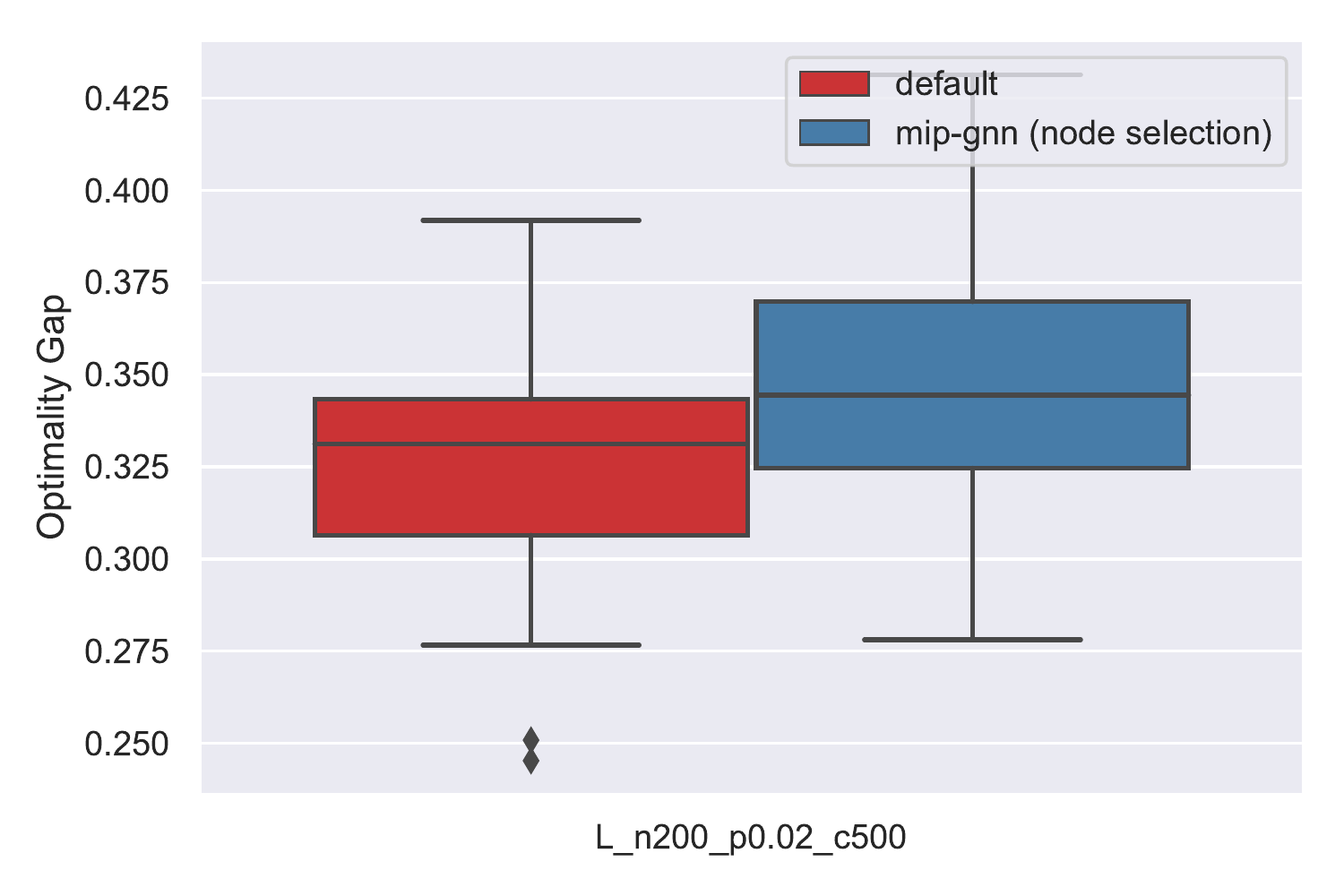}
\caption{Fixed-Charge Multi-Commodity Network Flow (FCMNF); Box plots for the distribution of \textbf{Optimality Gaps} for the ten problem sets, each with 100 instances; lower is better. 
}
\label{fig:fcmnf_main/fcmnf_box_gap}
\end{subfigure}
\vspace{10pt}
\caption{Fixed-Charge Multi-Commodity Network Flow (FCMNF); distribution of \textbf{Primal Integrals} and \textbf{Optimality Gaps}.}
\end{figure}
\newpage

\section{Additional experimental results -- GISP}
\label{app:gisp_allmethods}

Here, we report on additional results for the GISP datasets.

\begin{itemize}
    \item[--] Table~\ref{tab:gisp_wtl_bestval} shows that the proposed method (MIP-GNN for node selection) finds a better final solution than default CPLEX. This complements the results in the main text, which show improvements w.r.t.\@ the primal integral.
    
    \item[--] Tables~\ref{tab:gisp_primal_allmethods} and~\ref{tab:gisp_gap} show statistics for the primal integral and the optimality gap, respectively, for both the default solver and three uses of the MIP-GNN models.    
    \item[--] Figures~\ref{fig:gisp_allmethods/gisp_box_primal} and~\cref{fig:gisp_allmethods/gisp_box_gap} extend \cref{fig:gisp_box_primal} and~\cref{fig:gisp_box_gap} from the main text to include the two other uses of MIP-GNN, ``variable selection" and ``warmstart". While ``variable selection" sometimes outperforms ``default", ``node selection" is typically the winner.
    
    \item[--] Figures~\ref{fig:gisp_allmethods/gisp_box_numsols} and~\ref{fig:gisp_allmethods/gisp_box_numsols_lp} shed more light into how MIP-GNN uses affect the solution finding process in the MIP solver. ``node selection" finds more incumbent solutions (\cref{fig:gisp_allmethods/gisp_box_numsols}) than ``default", but also many more of those incumbents are integer solutions to node LP relaxations (\cref{fig:gisp_allmethods/gisp_box_numsols_lp}). This indicates that this guided node selection strategy is moving into more promising reasons of the search tree, which makes incumbents (i.e., improved integer-feasible solutions) more likely to be found by simply solving the node LP relaxations. In contrast, ``default" has to rely on solver heuristics to find incumbents, which may incur additional running time costs.
\end{itemize}

%\begin{wrapfigure}{l}{0.42\textwidth}
\begin{table}[!htbp]
\centering
\caption{Generalized Independent Set Problem; number of wins, ties, and losses, for \textbf{MIP-GNN (node selection)}, with respect to the \textbf{objective value of the best solution found} compared to the default solver setting. Column ``p-value" refers to the result of the Wilcoxon signed-rank test; smaller values indicate that our method is better than default.}
%{\arraystretch}{1.0}
\begin{tabular}{lrrrr}
\toprule
\textbf{Problem Set} & \textbf{Wins} & \textbf{Ties} & \textbf{Losses} &     \textbf{p-value} \\
\midrule
C125.9        &    2 &   98 &      0 &   0.08 \\
C250.9      &   89 &    1 &     10 & 2.3e-16 \\
brock200\_2     &   75 &    1 &     24 & 3.4e-08 \\
brock200\_4     &   87 &    3 &     10 & 2.2e-15 \\
gen200\_p0.9\_44 &   97 &    0 &      3 & 3.5e-18 \\
gen200\_p0.9\_55 &   99 &    0 &      1 & 2.7e-18 \\
hamming8-4     &   98 &    0 &      2 & 2.3e-18 \\
keller4      &   84 &    1 &     15 & 8.7e-13 \\
p\_hat300-1     &   86 &    0 &     14 & 2.7e-14 \\
p\_hat300-2    &   87 &    0 &     13 & 1e-15 \\
\bottomrule
\end{tabular}
\label{tab:gisp_wtl_bestval}
\end{table}
%\end{wrapfigure}

\begin{table}
\centering
\caption{Generalized Independent Set Problem; number of wins, ties, and losses, for \textbf{MIP-GNN (warmstart)}, with respect to the \textbf{objective value of the best solution found} compared to the default solver setting. Column ``p-value" refers to the result of the Wilcoxon signed-rank test; smaller values indicate that our method is better than default.}
\begin{tabular}{lrrrr}
\toprule
\textbf{Problem Set} & \textbf{Wins} & \textbf{Ties} & \textbf{Losses} &     \textbf{p-value} \\
\midrule
C125.9.clq         &    2 &   98 &      0 &   0.09 \\
C250.9.clq         &   52 &   25 &     23 & 2e-05 \\
brock200\_2.clq     &   42 &    6 &     52 &    0.96 \\
brock200\_4.clq     &   14 &   77 &      9 &    0.17 \\
gen200\_p0.9\_44.clq &   43 &   22 &     35 &     0.18 \\
gen200\_p0.9\_55.clq &   67 &   14 &     19 & 1.8e-07 \\
hamming8-4.clq     &   62 &   14 &     24 & 1.3e-07 \\
keller4.clq        &   36 &   47 &     17 &  5e-03 \\
p\_hat300-1.clq     &   89 &    0 &     11 & 9.8e-15 \\
p\_hat300-2.clq     &   87 &    0 &     13 & 1.2e-14 \\
\bottomrule
\end{tabular}
\label{tab:gisp_wtl_bestval_warmstart}
\end{table}

\begin{table}
\centering
\caption{Generalized Independent Set Problem; number of wins, ties, and losses, for \textbf{MIP-GNN (warmstart)}, with respect to the \textbf{optimality gap} compared to the default solver setting. Column ``p-value" refers to the result of the Wilcoxon signed-rank test; smaller values indicate that our method is better than default.}
\begin{tabular}{lrrrr}
\toprule
\textbf{Problem Set} & \textbf{Wins} & \textbf{Ties} & \textbf{Losses} &     \textbf{p-value} \\
\midrule
C125.9.clq         &    4 &   96 &      0 &   0.03 \\
C250.9.clq         &   59 &   17 &     24 & 5.4e-06 \\
brock200\_2.clq     &   25 &    0 &     75 &           1 \\
brock200\_4.clq     &   32 &   37 &     31 &    0.2 \\
gen200\_p0.9\_44.clq &   47 &    4 &     49 &    0.23 \\
gen200\_p0.9\_55.clq &   80 &    1 &     19 & 3.2e-10 \\
hamming8-4.clq     &   75 &    6 &     19 & 1.6e-10 \\
keller4.clq        &   65 &    9 &     26 & 7.6e-06 \\
p\_hat300-1.clq     &   96 &    0 &      4 & 4.1e-17 \\
p\_hat300-2.clq     &   67 &    0 &     33 & 3e-06 \\
\bottomrule
\end{tabular}
\label{tab:gisp_wtl_gap_warmstart}
\end{table}

\begin{figure}[!htbp]
    \centering
    \includegraphics[scale=0.35]{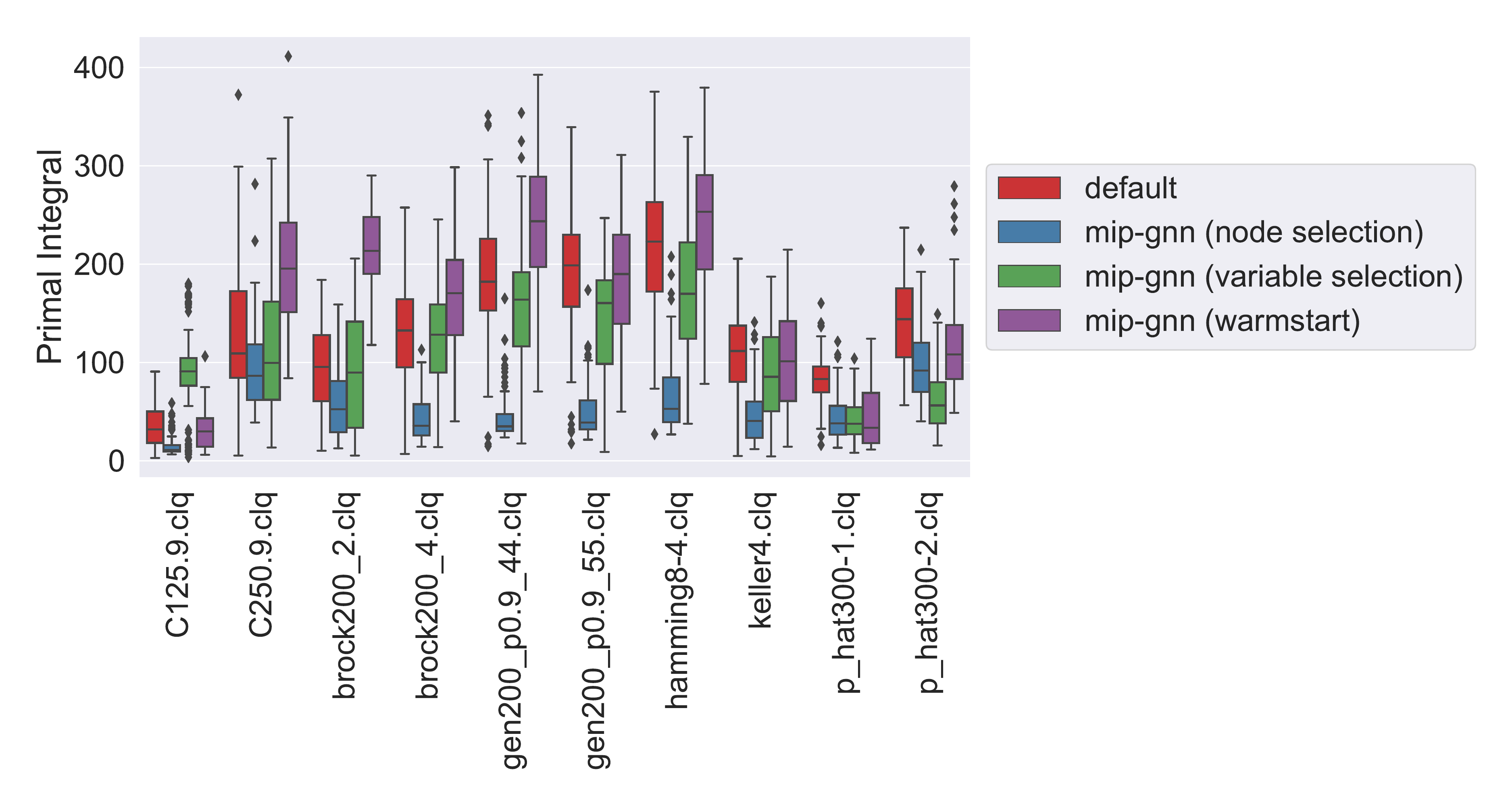}
    \caption{Generalized Independent Set Problem; Box plots for the distribution of \textbf{Primal Integrals} for the ten problem sets, each with 100 instances; lower is better. }
    \label{fig:gisp_allmethods/gisp_box_primal}
\end{figure}

\begin{figure}[!htbp]
    \centering
    \includegraphics[scale=0.35]{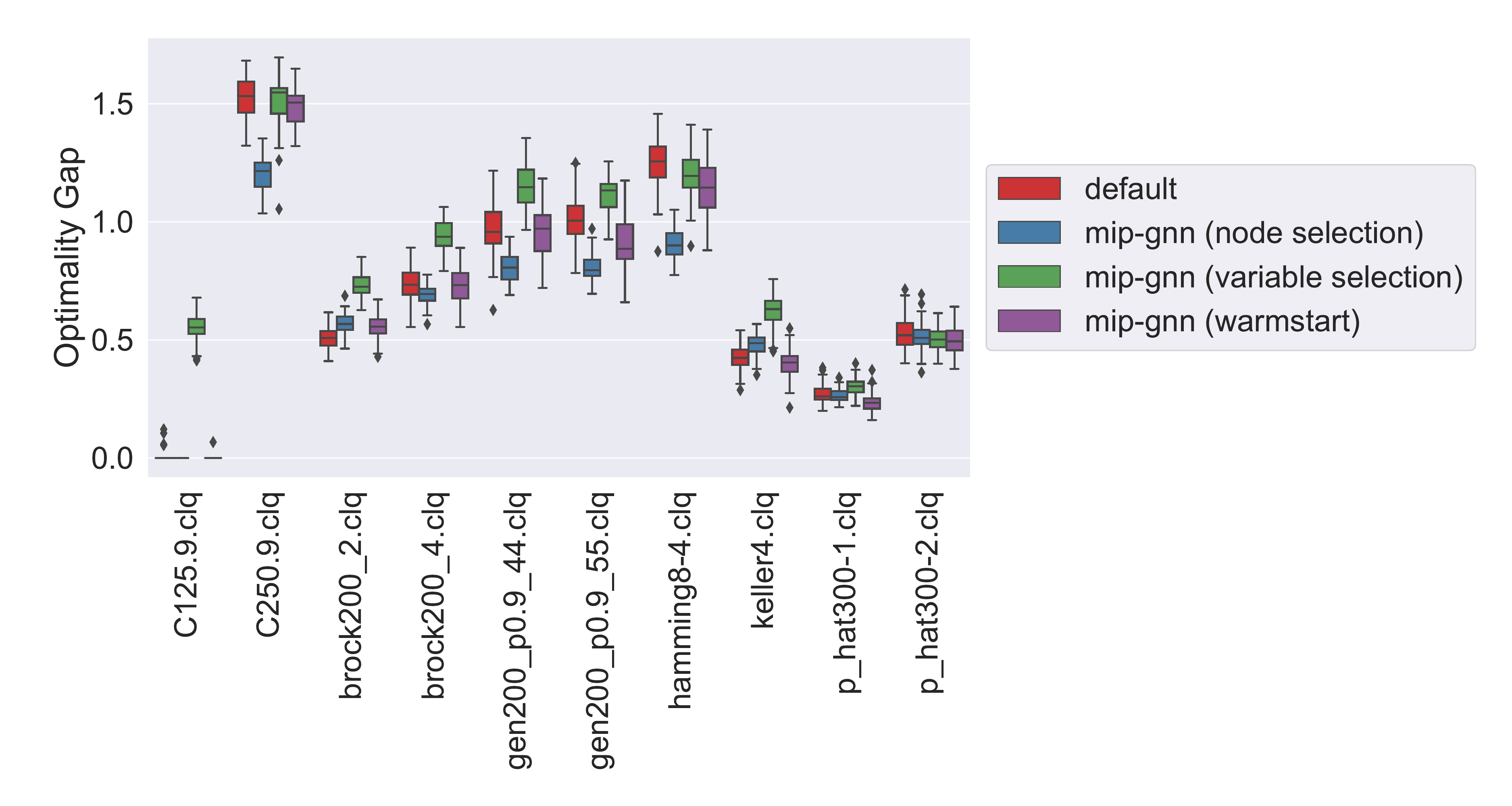}
    \caption{Generalized Independent Set Problem; box plots for the distribution of \textbf{Optimality Gaps} for the ten problem sets, each with 100 instances; lower is better. }
    \label{fig:gisp_allmethods/gisp_box_gap}
\end{figure}

\begin{figure}[!htbp]
    \centering
    \includegraphics[scale=0.35]{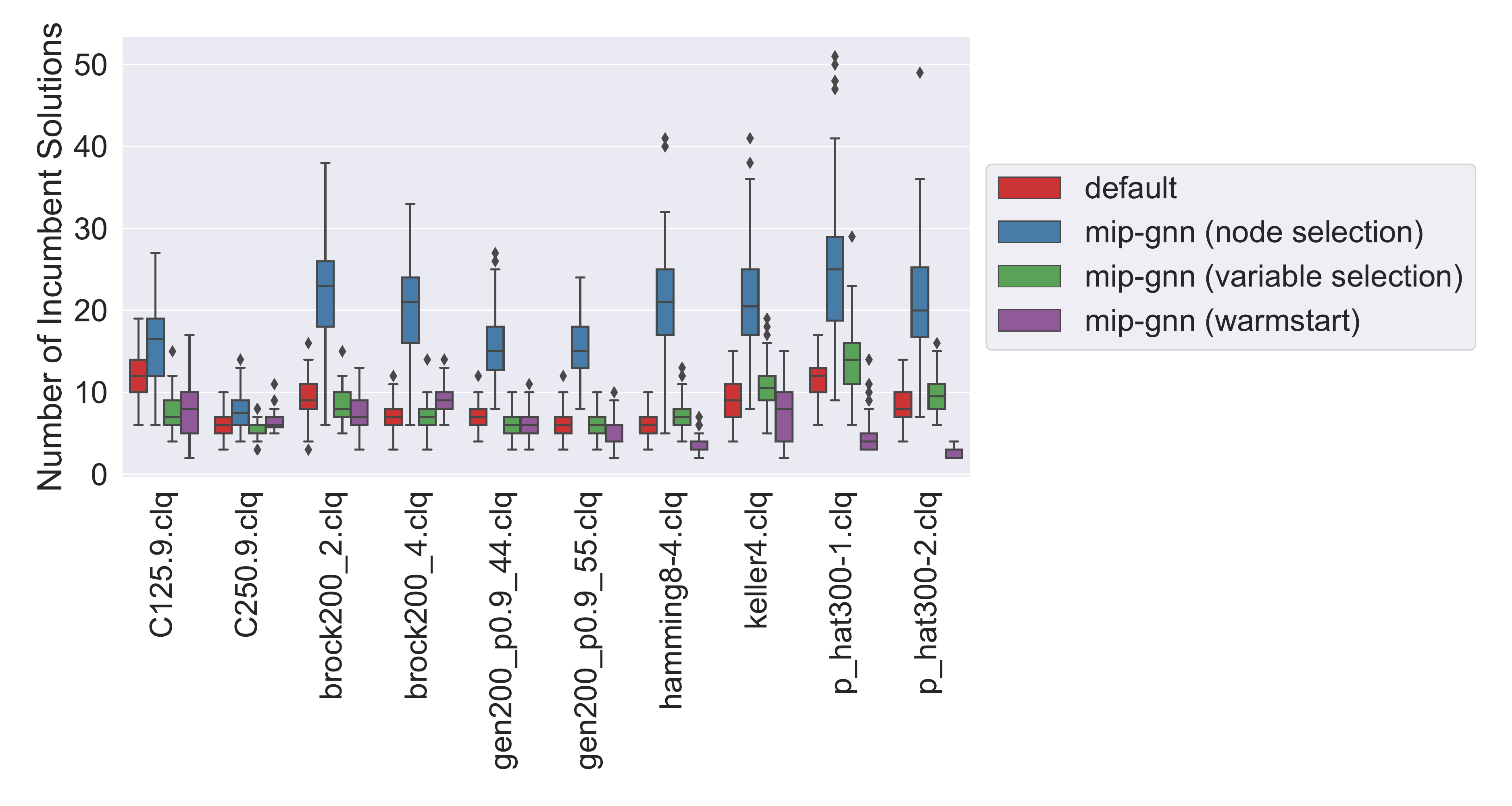}
    \caption{Generalized Independent Set Problem; box plots for the distribution of the \textbf{number of solutions found} during branch and bound for the ten problem sets, each with 100 instances.}
    \label{fig:gisp_allmethods/gisp_box_numsols}
\end{figure}

\begin{figure}[!htbp]
    \centering
    \includegraphics[scale=0.35]{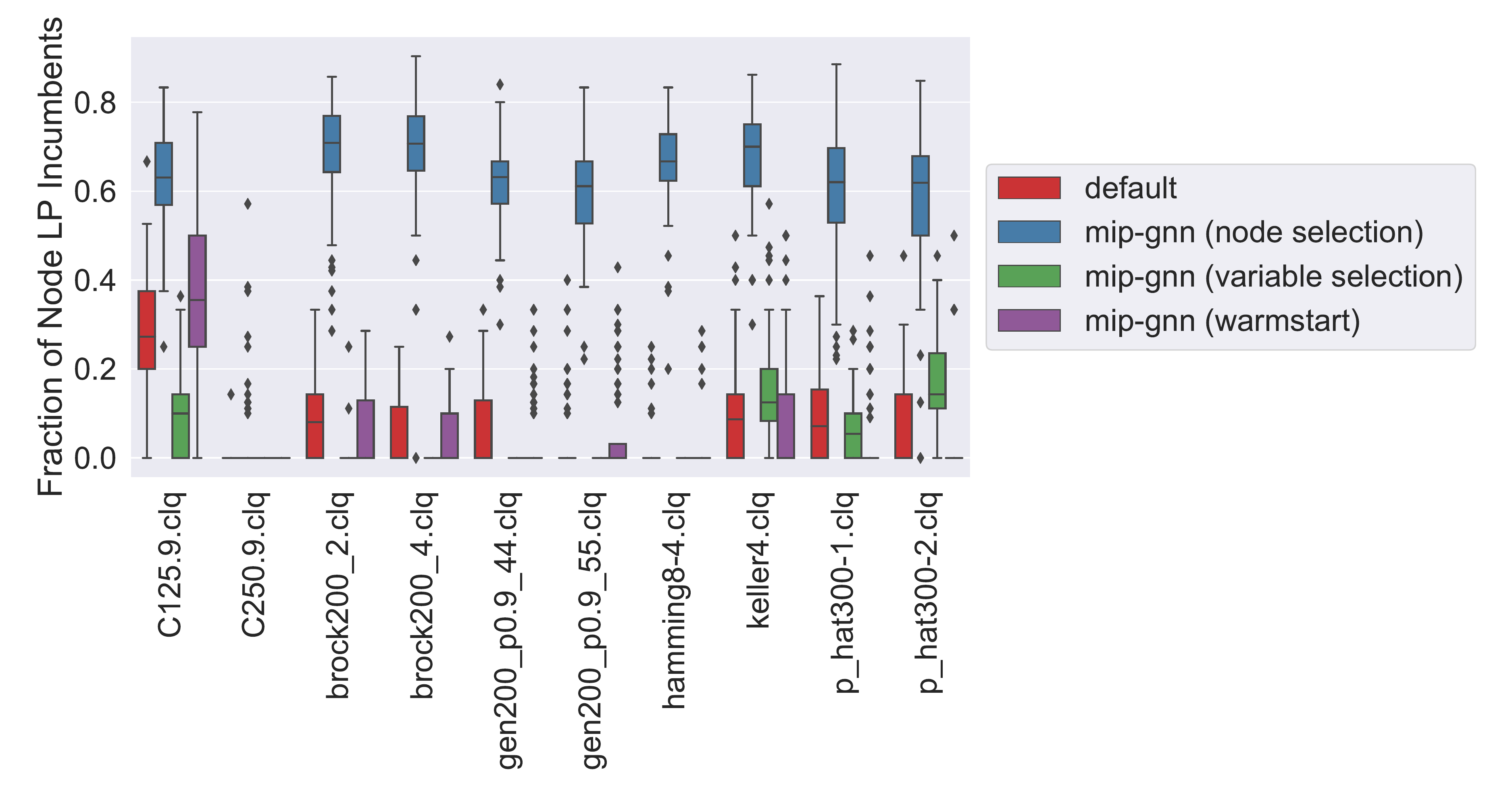}
    \caption{Generalized Independent Set Problem; box plots for the distribution of the \textbf{number of solutions found during branch and bound by solving the LP relaxation at a node of the search tree}; each boxplot corresponds to on of the ten problem sets, each with 100 instances.}
    \label{fig:gisp_allmethods/gisp_box_numsols_lp}
\end{figure}

\begin{table}[!htbp]
\centering
\caption{Generalized Independent Set Problem; statistics on the \textbf{Primal Integral} for the ten problem sets, each with 100 instances; lower is better.}
\begin{tabular}{llrrr}
\toprule
\multirow{3}{*}{\vspace*{4pt}\textbf{Problem Set}}& \multirow{3}{*}{\vspace*{4pt}\textbf{Method}} &\multicolumn{3}{c}{\textbf{Statistics}}\\
				\cmidrule{3-5}
               &             & \textbf{Avg.}          &  \textbf{Std.} & \textbf{Median}  \\
\midrule
C125.9.clq & default &  34.19 & 20.78 &  31.85 \\
               & MIP-GNN (node selection) &  14.90 &  9.76 &  11.13 \\
               & MIP-GNN (variable selection) &  89.55 & 50.28 &  91.07 \\
               & MIP-GNN (warmstart) &  30.73 & 18.87 &  29.85 \\\midrule
C250.9.clq & default & 129.32 & 72.61 & 109.21 \\
               & MIP-GNN (node selection) &  96.31 & 42.92 &  86.54 \\
               & MIP-GNN (variable selection) & 116.01 & 67.90 &  99.63 \\
               & MIP-GNN (warmstart) & 200.62 & 65.49 & 195.52 \\\midrule
brock200\_2.clq & default &  96.95 & 43.53 &  95.58 \\
               & MIP-GNN (node selection) &  59.22 & 34.82 &  52.44 \\
               & MIP-GNN (variable selection) &  89.95 & 55.09 &  89.55 \\
               & MIP-GNN (warmstart) & 214.25 & 40.62 & 213.43 \\\midrule
brock200\_4.clq & default & 130.17 & 52.64 & 132.57 \\
               & MIP-GNN (node selection) &  44.49 & 24.60 &  35.67 \\
               & MIP-GNN (variable selection) & 128.70 & 49.62 & 128.40 \\
               & MIP-GNN (warmstart) & 166.68 & 52.80 & 170.40 \\\midrule
gen200\_p0.9\_44.clq & default & 181.35 & 65.72 & 182.17 \\
               & MIP-GNN (node selection) &  43.75 & 23.18 &  34.98 \\
               & MIP-GNN (variable selection) & 165.85 & 63.93 & 163.84 \\
               & MIP-GNN (warmstart) & 239.93 & 68.44 & 243.68 \\\midrule
gen200\_p0.9\_55.clq & default & 186.81 & 63.17 & 198.87 \\
               & MIP-GNN (node selection) &  50.97 & 28.25 &  38.78 \\
               & MIP-GNN (variable selection) & 148.54 & 56.96 & 160.39 \\
               & MIP-GNN (warmstart) & 184.92 & 59.18 & 189.92 \\\midrule
hamming8-4.clq & default & 220.08 & 66.47 & 222.97 \\
               & MIP-GNN (node selection) &  65.05 & 36.29 &  52.95 \\
               & MIP-GNN (variable selection) & 173.04 & 66.91 & 169.89 \\
               & MIP-GNN (warmstart) & 248.14 & 63.11 & 253.16 \\\midrule
keller4.clq & default & 107.92 & 43.92 & 111.63 \\
               & MIP-GNN (node selection) &  46.99 & 29.44 &  40.66 \\
               & MIP-GNN (variable selection) &  90.94 & 46.81 &  85.39 \\
               & MIP-GNN (warmstart) &  99.43 & 50.20 & 101.27 \\\midrule
p\_hat300-1.clq & default &  83.62 & 24.00 &  83.06 \\
               & MIP-GNN (node selection) &  43.05 & 22.48 &  37.97 \\
               & MIP-GNN (variable selection) &  41.86 & 21.70 &  37.60 \\
               & MIP-GNN (warmstart) &  44.95 & 31.08 &  33.41 \\\midrule
p\_hat300-2.clq & default & 139.69 & 43.99 & 143.98 \\
               & MIP-GNN (node selection) &  96.26 & 36.65 &  91.61 \\
               & MIP-GNN (variable selection) &  62.58 & 34.09 &  56.30 \\
               & MIP-GNN (warmstart) & 117.24 & 44.95 & 108.24 \\
\bottomrule
\end{tabular}
\label{tab:gisp_primal_allmethods}
\end{table}

\begin{table}
\centering
\caption{Generalized Independent Set Problem; statistics on the \textbf{Optimality Gap} for the ten problem sets, each with 100 instances; lower is better. The gap is calculated at solver termination with a time limit of 30 minutes per instance. A gap of zero indicates that an instance has been solved to proven optimality. For GISP, this only happens for (most methods and most instances of) C125.9.clq, the problem set with the smallest number of variables and constraints. For other problem sets, most instances time out with relatively large gaps. A gap of $0.4$, for example, can be interpreted as: the best solution found at termination is guaranteed to be with in $40\%$ of optimal. The optimality gap is calculated by CPLEX as  $\nicefrac{|\text{bestbound}-\text{bestinteger}|}{(10^{-9}+|\text{bestinteger}|)}$.}
\label{tab:gisp_gap}
\begin{tabular}{llrrr}
\toprule
\multirow{3}{*}{\vspace*{4pt}\textbf{Problem Set}}& \multirow{3}{*}{\vspace*{4pt}\textbf{Method}} &\multicolumn{3}{c}{\textbf{Statistics}}\\
				\cmidrule{3-5}
               &             & \textbf{Avg.}          &  \textbf{Std.} & \textbf{Median}  \\
\midrule
C125.9.clq & default & 0.00 & 0.02 &   0.00 \\
               & MIP-GNN (node selection) & 0.00 & 0.00 &   0.00 \\
               & MIP-GNN (variable selection) & 0.55 & 0.06 &   0.55 \\
               & MIP-GNN (warmstart) & 0.00 & 0.01 &   0.00 \\\midrule
C250.9.clq & default & 1.53 & 0.09 &   1.53 \\
               & MIP-GNN (node selection) & 1.21 & 0.07 &   1.21 \\
               & MIP-GNN (variable selection) & 1.52 & 0.10 &   1.55 \\
               & MIP-GNN (warmstart) & 1.48 & 0.08 &   1.50 \\\midrule
brock200\_2.clq & default & 0.51 & 0.05 &   0.51 \\
               & MIP-GNN (node selection) & 0.57 & 0.04 &   0.57 \\
               & MIP-GNN (variable selection) & 0.73 & 0.05 &   0.73 \\
               & MIP-GNN (warmstart) & 0.56 & 0.05 &   0.56 \\\midrule
brock200\_4.clq & default & 0.73 & 0.07 &   0.73 \\
               & MIP-GNN (node selection) & 0.69 & 0.04 &   0.69 \\
               & MIP-GNN (variable selection) & 0.94 & 0.06 &   0.94 \\
               & MIP-GNN (warmstart) & 0.73 & 0.07 &   0.73 \\\midrule
gen200\_p0.9\_44.clq & default & 0.97 & 0.10 &   0.96 \\
               & MIP-GNN (node selection) & 0.81 & 0.06 &   0.81 \\
               & MIP-GNN (variable selection) & 1.15 & 0.08 &   1.15 \\
               & MIP-GNN (warmstart) & 0.96 & 0.10 &   0.97 \\\midrule
gen200\_p0.9\_55.clq & default & 1.01 & 0.10 &   1.01 \\
               & MIP-GNN (node selection) & 0.81 & 0.05 &   0.80 \\
               & MIP-GNN (variable selection) & 1.12 & 0.07 &   1.13 \\
               & MIP-GNN (warmstart) & 0.91 & 0.10 &   0.89 \\\midrule
hamming8-4.clq & default & 1.25 & 0.10 &   1.26 \\
               & MIP-GNN (node selection) & 0.91 & 0.06 &   0.90 \\
               & MIP-GNN (variable selection) & 1.20 & 0.10 &   1.19 \\
               & MIP-GNN (warmstart) & 1.15 & 0.11 &   1.14 \\\midrule
keller4.clq & default & 0.42 & 0.05 &   0.42 \\
               & MIP-GNN (node selection) & 0.48 & 0.04 &   0.49 \\
               & MIP-GNN (variable selection) & 0.62 & 0.07 &   0.63 \\
               & MIP-GNN (warmstart) & 0.40 & 0.06 &   0.40 \\\midrule
p\_hat300-1.clq & default & 0.27 & 0.04 &   0.26 \\
               & MIP-GNN (node selection) & 0.26 & 0.03 &   0.26 \\
               & MIP-GNN (variable selection) & 0.30 & 0.03 &   0.30 \\
               & MIP-GNN (warmstart) & 0.23 & 0.04 &   0.23 \\\midrule
p\_hat300-2.clq & default & 0.53 & 0.06 &   0.52 \\
               & MIP-GNN (node selection) & 0.51 & 0.05 &   0.51 \\
               & MIP-GNN (variable selection) & 0.50 & 0.05 &   0.50 \\
               & MIP-GNN (warmstart) & 0.50 & 0.06 &   0.49 \\
\bottomrule
\end{tabular}
\end{table}

\begin{table}[!htbp]
\centering
\caption{Training dataset statistics.}
\label{tab:ds_stats}
\resizebox{1.0\textwidth}{!}{
\begin{tabular}{lrrrr}
\toprule
\textbf{Problem Set}  & \textbf{Num. graphs} & \textbf{Avg. num. variables} & \textbf{Avg. num. constraints} & \textbf{Avg. non-zeros in constraints}  \\
\midrule
p\_hat300-2.clq & 1\,000  &  16\,748.8 & 21\,928.0 & 120\,609.6       \\
gisp\_C250.9.clq & 1\,000  & 21\,239.5 & 27\,984.0 & 153\,915.0 \\
keller4.clq &  1\,000  &   7\,247.6 & 9\,435.0 & 51\,893.3 \\
hamming8-4.clq & 1\,000  &  15\,906.6 & 20\,864.0 & 114\,757.1 \\
gen200\_p0.9\_55.clq & 1\,000  & 13\,634.5 & 17\,910.0 & 98\,508.9 \\
gen200\_p0.9\_44.clq & 1\,000  & 13\,643.8 & 17\,910.0 & 98\,527.5    \\
C125.9.clq & 1\,000  &  5\,347.9 & 6\,963.0 & 38\,297.9  \\
p\_hat300-1.clq & 1\,000  &  8\,501.1 & 10\,933.0 & 60\,134.1  \\
brock200\_4.clq & 1\,000  & 10\,018.5 & 13\,089.0 & 71\,993.1 \\
brock200\_2.clq & 1\,000  &  7\,607.9 & 9\,876.0 & 54\,319.8
 \\
L\_n200\_p0.02\_c500 & 1\,000  & 80\,497.0 & 20\,797.0 & 479\,794.0 \\
\bottomrule
\end{tabular}}
\end{table}

\clearpage
\section{Additional Experimental Results -- GISP on Erdos-Renyi Graphs}

We also performed experiments on GISP instances based on synthetic Erdos-Renyi graphs. Contrary to other datasets, each generated instance is based on a \textit{different} randomly generated Erdos-Renyi graph on 200 nodes, with the number of edges ranging between 1\,867 and 2\,136; the Erdos-Renyi graph generator has an edge probability of 0.1, and so the expected number of edges is roughly 
2\,000. As in the other GISP graphs used in the paper, the set of removable edges is also random with the same probability of inclusion in the set, $\alpha=0.75$.

Of note is the fact that these instances are generally easier to solve to optimality than other datasets considered in this work. As such, the solver time limit is reduced to 10 minutes. Also, the room for improvement in metrics of interest such as the Primal Integral is smaller. Coupled with the fact that each GISP instance has a different underlying random graph, makes improving over default CPLEX a challenging task for our method.

On 956 out of 1\,000 unseen test instances from this alternative distribution, MIP-GNN (node selection) produces a smaller primal integral than default CPLEX; the average primal integral is halved using MIP-GNN (node selection).

\begin{table}[h]
\centering
\label{tab:_wtl_primal_integral_gisp_er}
\caption{GISP on Erdos-Renyi graphs; number of wins, ties, and losses, for MIP-GNN (node selection), with respect to the \textbf{Primal Integral} compared to the default solver setting.}
\begin{tabular}{lrrrr}
\toprule
\textbf{Problem Set} & \textbf{Wins} & \textbf{Ties} & \textbf{Losses} &     \textbf{p-value} \\
\midrule
er\_200\_SET2\_1k &  956 &    0 &     44 & 3e-157 \\
\bottomrule
\end{tabular}
\end{table}

\begin{figure}[!htbp]
    \centering
    \includegraphics[scale=0.8]{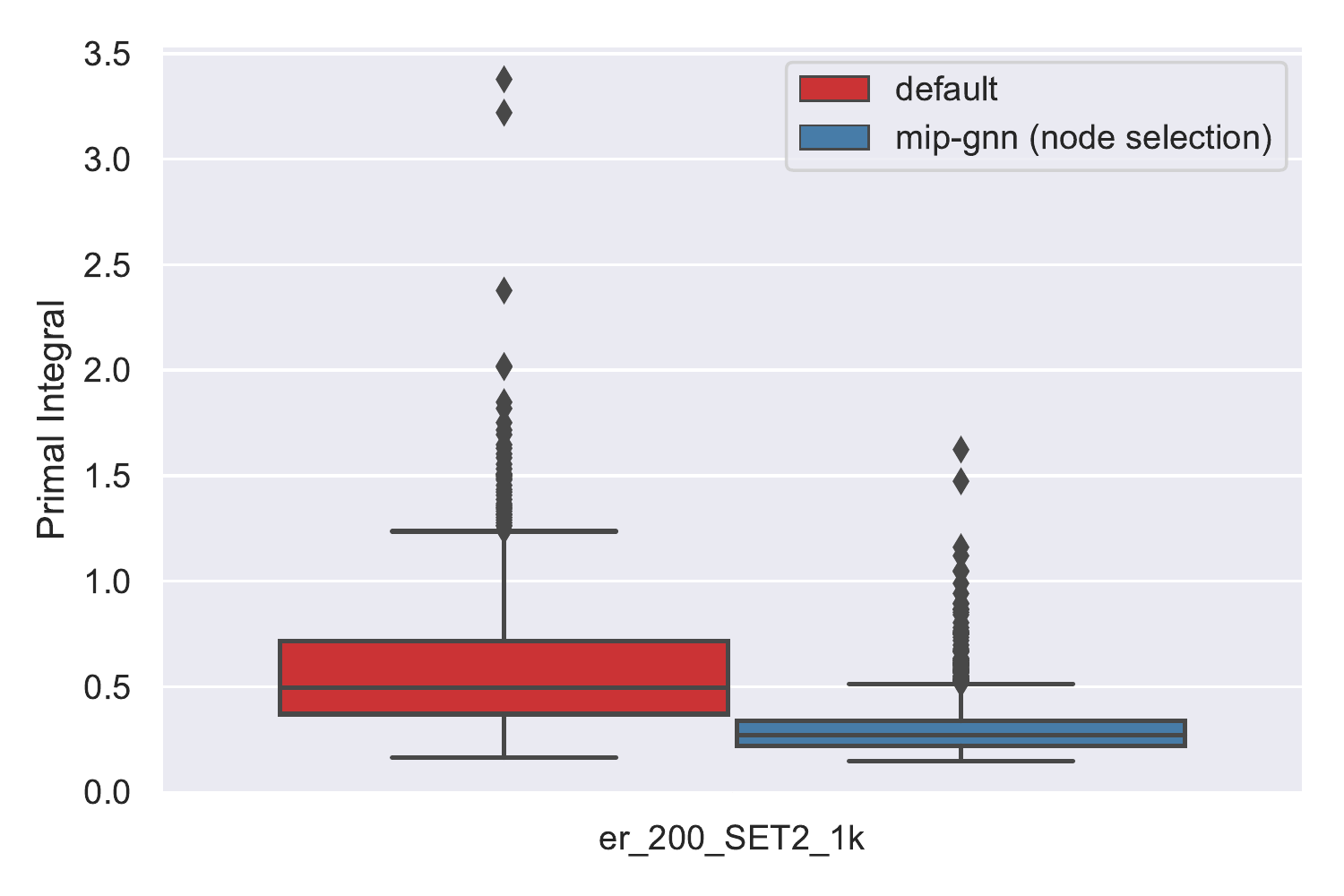}
    \caption{Box plots for the distribution of \textbf{Primal Integrals} for the Erdos-Renyi test dataset with 1000 instances; lower is better. }
    \label{fig:gisp_allmethods/gisp_box_primal}
\end{figure}

\end{document}